\documentclass{article}

\usepackage[final]{nips_2018}

\usepackage[utf8]{inputenc} 
\usepackage[T1]{fontenc}    
\usepackage{hyperref}       
\usepackage{url}            
\usepackage{booktabs}       
\usepackage{amsfonts}       
\usepackage{nicefrac}       
\usepackage{microtype}      

\usepackage{enumitem}

\usepackage{color}
\usepackage{algorithm}
\usepackage{algorithmic}
\usepackage{amsmath}
\usepackage{amssymb}
\usepackage{amsthm}
\usepackage{comment}
\usepackage{graphicx}
\usepackage{bbm}
\usepackage{url}
\usepackage{xr}
\usepackage{subfigure}

\usepackage{listings}
\usepackage{graphicx}
\usepackage{xspace}

\usepackage{xcolor}
\definecolor{DarkGreen}{rgb}{0.1,0.5,0.1}
\newcommand{\todo}[1]{\textbf{\textcolor{DarkGreen}{TODO: #1}}}
\newcommand{\jda}[1]{\textbf{\textcolor{blue}{JDA: #1}}}
\newcommand{\jk}[1]{\textbf{\textcolor{red}{JK: #1}}}

\def\g{\gamma_{\K}}

\def\K{\mathcal{K}}

\def\reals{\mathbb{R}}

\def\argmin{\mathop{\textnormal{argmin}}}
\def\argmax{\mathop{\textnormal{argmax}}}

\def\balpha{\boldsymbol{\alpha}}
\def\alg{\text{OAlg}}

\def\FTL{\textsc{FollowTheLeader}\xspace}
\def\FTRL{\textsc{FollowTheRegularizedLeader}\xspace}
\def\FTPL{\textsc{FollowThePerturbedLeader}\xspace}
\def\BTRL{\textsc{BeTheRegularizedLeader}\xspace}

\def\OFTRL{\textsc{OptimisticFTRL}\xspace}

\def\OFTL{\textsc{OptimisticFTL}\xspace}
\def\MD{\textsc{MirrorDescent}\xspace}
\def\HB{\textsc{HeavyBall}\xspace}
\def\FW{\textsc{FrankWolfe}\xspace}
\def\NA{\textsc{NesterovAcceleration}\xspace}

\newcommand{\lr}[2]{\left\langle#1,#2\right\rangle}

\newcommand{\regret}[1]{\balpha\textsc{-Reg}^{#1}}
\newcommand{\avgregret}[1]{\overline{\balpha\textsc{-Reg}}^{#1}}
\newcommand{\XX}{\mathcal{X}}
\newcommand{\YY}{\mathcal{Y}}

\newcommand{\pr}[1]{\left(#1\right)}

\def\yof{\widetilde{y}}
\def\yftl{\hat{y}}
\def\xof{\widetilde{x}}
\def\xav{\bar{x}}
\def\yav{\bar{y}}

\newtheorem{lemma}{Lemma}
\newtheorem{corollary}{Corollary}
\newtheorem{theorem}{Theorem}

\title{Acceleration through Optimistic No-Regret Dynamics}

\author{
  Jun-Kun Wang \\
  College of Computing\\
  Georgia Institute of Technology\\
  Atlanta, GA 30313 \\
  \texttt{jimwang@gatech.edu} \\
  \And
  Jacob Abernethy \\
  College of Computing\\
  Georgia Institute of Technology\\
  Atlanta, GA 30313 \\
  \texttt{prof@gatech.edu} \\
}

\begin{document}

\maketitle

\begin{abstract}
We consider the problem of minimizing a smooth convex function by reducing the optimization to computing the Nash equilibrium of a particular zero-sum convex-concave game. Zero-sum games can be solved using online learning dynamics, where a classical technique involves simulating two no-regret algorithms that play against each other and, after $T$ rounds, the average iterate is guaranteed to solve the original optimization problem with error decaying as $O(\log T/T)$.
In this paper we show that the technique can be enhanced to a rate of $O(1/T^2)$ by extending recent work \cite{RS13,SALS15} that leverages \textit{optimistic learning} to speed up equilibrium computation.
The resulting optimization algorithm derived from this analysis coincides \textit{exactly} with the well-known \NA \cite{N83a} method, and indeed the same story allows us to recover several variants of the Nesterov's algorithm via small tweaks. We are also able to establish the accelerated linear rate for a function which is both strongly-convex and smooth. This methodology unifies a number of different iterative optimization methods: we show that the \HB algorithm is precisely the non-optimistic variant of \NA, and recent prior work already established a similar perspective on \FW \cite{AW17,ALLW18}.
\end{abstract}

\section{Introduction}

One of the most successful and broadly useful tools recently developed within the machine learning literature is the \textit{no-regret framework}, and in particular \textit{online convex optimization} (OCO) \cite{Z03}. In the standard OCO setup, a learner is presented with a sequence of (convex) loss functions $\ell_1(\cdot), \ell_2(\cdot), \ldots$, and must make a sequence of decisions $x_1, x_2, \ldots$ from some set $\K$ in an online fashion, and observes $\ell_t$ after only having committed to $x_t$. Assuming the sequence $\{ \ell_t \}$ is chosen by an adversary, the learner aims is to minimize the \textit{average regret} $\bar R_T := \frac 1 T \left( \sum_{t=1}^T \ell_t(x_t) - \min_{x \in \K} \sum_{t=1}^T \ell_t(x) \right)$ against any such loss functions. Many simple algorithms have been developed for OCO problems---including \MD, \FTRL, \FTPL, etc.---and these algorithms exhibit regret guarantees that are strong even against adversarial opponents. Under very weak conditions one can achieve a regret rate of $\bar R_T = O(1/\sqrt T)$, or even $\bar R_T = O(\log T/T)$ with required curvature on $\ell_t$.

One can apply online learning tools to several problems, but perhaps the simplest is to find the approximate minimum of a convex function $\argmin_{x \in \K} f(x)$. With a simple reduction we set $\ell_t = f$, and it is easy to show that, via Jensen's inequality,  the average iterate $\bar x_T := \frac{x_1 + \ldots + x_T}{T}$ satisfies
\[
 \textstyle f(\bar x_T) \leq \frac 1 T \sum_{t=1}^T f(x_t) = \frac 1 T \sum_{t=1}^T \ell_t(x_t) \leq \min_{x \in \K} \frac 1 T \sum_{t=1}^T \ell_t(x) + \bar R_T = \min_{x \in \K} f(x) + \bar R_T
\]
hence $\bar R_T$ upper bounds the approximation error. But this reduction, while simple and natural, is quite limited. For example, we know that when $f(\cdot)$ is \textit{smooth}, more sophisticated algorithms such as \FW and \HB achieve convergence rates of $O(1/T)$, whereas the now-famous \NA algorithm achieves a rate of $O(1/T^2)$. The fast rate shown by Nesterov was quite surprising at the time, and many researchers to this day find the result quite puzzling. There has been a great deal of work aimed at providing a more natural explanation of acceleration, with a more intuitive convergence proof \cite{wibisono2016variational,AO17,FB15}. This is indeed one of the main topics of the present work, and we will soon return to this discussion.

Another application of the no-regret framework is the solution of so-called saddle-point problems, which are equivalently referred to as Nash equilibria for zero-sum games. Given a function $g(x,y)$ which is convex in $x$ and concave in $y$ (often called a \textit{payoff function}),
define $V^*=\inf_{x \in \K} \sup_{y} g(x,y)$. An $\epsilon$-\textit{equilibrium} of $g(\cdot, \cdot)$ is a pair $\hat x, \hat y$ such that such that
\begin{equation}
\textstyle V^* - \epsilon \leq \inf_{x \in \K} g(x, \hat y)  \leq V^* \leq   \sup_y g(\hat x, y) \leq V^* + \epsilon.
\end{equation}

One can find an approximate saddle point of the game with the following setup:  implement a no-regret learning algorithm for both the $x$ and $y$ players simultaneously, after observing the actions $\{ x_t, y_t \}_{t=1\ldots T}$ return the time-averaged iterates $(\hat x, \hat y) = \left(\frac{x_1 + \ldots + x_T}{T}, \frac{y_1 + \ldots + y_T}{T}\right)$. 
A simple proof shows that $(\hat x, \hat y)$ is an approximate equilibrium, with approximation bounded by the average regret of both players (see Theorem~\ref{thm:convergence}). In the case where the function $g(\cdot, \cdot)$ is biaffine, the no-regret reduction guarantees a rate of $O(1/\sqrt{T})$, and it was assumed by many researchers this was the fastest possible using this framework. But one of the most surprising online learning results to emerge in recent years established that no-regret dynamics can obtain an even faster rate of $O(1/T)$. Relying on tools developed by \cite{CJ12}, this fact was first proved by \cite{RK13} and extended by \cite{SALS15}. The new ingredient in this recipe is the use of \textit{optimistic} learning algorithms, where the learner seeks to benefit from the predictability of slowly-changing inputs $\{ \ell_t \}$.

We will consider solving the classical convex optimization problem $\min_x f(x)$, for smooth functions $f$, by instead solving an associated saddle-point problem which we call the \textit{Fenchel Game}.
Specifically, we consider that the payoff function $g$ of the game to be
\begin{equation} \label{eq:fenchelgame}
\textstyle{ g(x,y) = \langle x, y \rangle - f^*(y). }
\end{equation}
where $f^*(\cdot)$ is the \textit{fenchel conjugate} of $f(\cdot)$. This is an appropriate choice of payoff function since, 
$V^* = \min_x f(x)$ and 
$\sup_y g(\hat{x}, y ) = \sup_{y} \langle \hat{x}, y \rangle - f^*(y) = f(\hat{x})$.
Therefore, by the definition of an $\epsilon$-equilibrium, we have that
\begin{lemma} \label{lem:fenchelgame}
  If $(\hat x, \hat y)$ is an $\epsilon$-equilibrium of the Fenchel Game \eqref{eq:fenchelgame}, then $f(\hat x) - \min_{x} f(x) \leq \epsilon$.
\end{lemma}

One can imagine computing the equilibrium of the Fenchel game using no-regret dynamics, and indeed this was the result of recent work \cite{AW17} establishing the \FW algorithm as precisely an instance of two competing learning algorithms. 

In the present work we will take this approach even further. 
\begin{enumerate}[topsep=0pt,itemsep=-1ex,partopsep=1ex,parsep=1ex]
  \item We show that, by considering a notion of \textit{weighted regret}, we can compute equilibria in the Fenchel game at a rate of $O(1/T^2)$ using no-regret dynamics where the only required condition is that $f$ is smooth. This improves upon recent work \cite{ALLW18} on a faster \FW method, which required strong convexity of $f$ (see Appendix~\ref{app:accFW}).
  \item We show that the secret sauce for obtaining the fast rate is precisely the use of an optimistic no-regret algorithm, \OFTL \cite{ALLW18}, combined with appropriate weighting scheme.
  \item We show that, when viewed simply as an optimization algorithm, this method is \textit{identically} the original \NA method. In addition, we recover several variants of \NA (see \cite{N83b,N88,N05}) using small tweaks of the framework.
  \item We show that if one simply plays \FTL without optimism, the resulting algorithm is precisely the \HB. The latter is known to achieve a suboptimal rate in general, and our analysis sheds light on this difference.
  \item Under the additional assumption that function $f(\cdot)$ is strongly convex, we show that an accelerated linear rate can also be obtained from the game framework.
  \item Finally, we show that the same equilibrium framework can also be extended to composite 
  optimization and lead to a variant of Accelerated Proximal Method. 
\end{enumerate}



\textbf{Related works:}
In recent years, there are growing interest in giving new interpretations of Nesterov’s accelerated algorithms. For example, \cite{T08} gives a unified analysis for some Nesterov’s accelerated algorithms
\cite{N88,N04,N05}, using the standard techniques and analysis in optimization literature.
\cite{LRP16} connects the design of accelerated algorithms with dynamical systems and control theory. \cite{BLS15} gives a geometric interpretation of the Nesterov’s method for unconstrained optimization, inspired by the ellipsoid method.
\cite{FB15} studies the Nesterov’s methods and the \HB method for quadratic non-strongly convex problems by analyzing the eigen-values of some linear dynamical systems.
\cite{AO17} proposes a variant of accelerated algorithms by mixing the updates of gradient descent and mirror descent and showing the updates are complementary.
\cite{SBC14,wibisono2016variational} connect the acceleration algorithms with differential equations.
In recent years there has emerged a lot of work where learning problems are treated as repeated games \cite{NIPS2013_5148,abernethy2008optimal}, and many researchers have been studying the relationship between game dynamics and provable convergence rates \cite{balduzzi2018mechanics,gidel2018negative,daskalakis2017training}.

We would like to acknowledge George Lan for his excellent notes titled ``Lectures on Optimization for Machine Learning'' (unpublished). In parallel to the development of the results in this paper, we discovered that Lan had observed a similar connection between \NA and repeated game playing (Chapter 3.4). A game interpretation was given by George Lan and Yi Zhou in Section 2.2 of \cite{LZ17}.

\section{Preliminaries} \label{pre}

\paragraph{Convex functions and conjugates.}
A function $f$ on $\reals^d$ is $L$-smooth w.r.t. a norm $\| \cdot \|$ if $f$ is everywhere differentiable and
it has lipschitz continuous gradient
$\| \nabla f(u) - \nabla f(v) \|_* \leq L \| u - v\|$,
where $\| \cdot \|_{*}$ denotes the dual norm.
Throughout the paper, our goal will be to solve the problem of minimizing an $L$-smooth function $f(\cdot)$ over a convex set $\K$. 
We also assume that the optimal solution of $x^* := \argmin_{x \in \K} f(x)$ has  finite norm.
For any convex function $f$, its Fenchel conjugate is $f^*(y) := \sup_{x \in \text{dom}(f) }  \langle x, y \rangle - f(x)$. 
If a function $f$ is convex, then its conjugate $f^*$ is also convex.
Furthermore,
when the function $f(\cdot)$ is strictly convex,
we have that $\nabla f(x) = \displaystyle \argmax_{y}  \langle x, y \rangle - f^*(y) $. 

Suppose we are given a differentiable function $\phi(\cdot)$, then 
the Bregman divergence $V_c(x)$ with respect to $\phi(\cdot)$ at 
a point $c$ is defined as $V_{c}(x):= \phi(x) - \langle \nabla \phi(c), x - c  \rangle  - \phi(c)$. Let $\|\cdot\|$ be any norm on $\reals^d$. When we have that $V_{c}(x) \geq \frac{\sigma}{2} \| c - x \|^{2}$ for any $x,c \in \text{dom}(\phi)$, we say that $\phi(\cdot)$ is a \emph{$\sigma$-strongly convex function} with respect to $\| \cdot \|$. Throughout the paper we assume that $\phi(\cdot)$ is 1-strongly convex.

\paragraph{No-regret zero-sum game dynamics.} Let us now consider the process of solving a zero-sum game via repeatedly play by a pair of online learning strategies. The sequential procedure is described in Algorithm~\ref{alg:game}. 

\begin{algorithm} 
   \caption{ Computing equilibrium using no-regret algorithms } \label{alg:game}
\begin{algorithmic}[1]
\STATE Input: sequence $\alpha_1, \ldots, \alpha_T > 0$
\FOR{$t= 1, 2, \dots, T$}
\STATE $y$-player selects $y_t\in \YY = \reals^d$ by $\alg^y$.
\STATE $x$-player selects $x_t \in \XX$ by $\alg^x$, possibly with knowledge of $y_t$.
\STATE $y$-player suffers loss $\ell_{t}(y_t)$ with weight $\alpha_t$, where $\ell_t(\cdot) = -g(x_t,\cdot)$.
\STATE $x$-player suffers loss $h_{t}(x_t)$ with weight $\alpha_t$, where $h_t(\cdot) = g(\cdot,y_t)$.
\ENDFOR
\STATE Output $(\xav_T,\yav_T) := \left(\frac{ \sum_{s=1}^T \alpha_s x_s  }{ A_T }, \frac{ \sum_{s=1}^T \alpha_s y_s  }{ A_T }\right)$.
\end{algorithmic}
\end{algorithm}

In this paper, we consider Fenchel game with weighted losses depicted in Algorithm~\ref{alg:game}, following the same setup as \cite{ALLW18}. In this game, the $y$-player plays before the $x$-player plays
and the $x$-player sees what the $y$-player plays before choosing its action.
The $y$-player receives loss functions $\alpha_{t} \ell_{t}(\cdot)$ in round $t$, in which $\ell_{t}(y):= f^{*}(y)- \langle x_t, y \rangle $,
while the x-player see its loss functions $\alpha_{t} h_{t}(\cdot)$ in round $t$, in which $h_{t}(x):= \langle x, y_t \rangle - f^{*}(y_t)$.
Consequently, we can define the \textit{weighted regret} of the $x$ and $y$ players as
\begin{eqnarray}
  \label{eq:yregret}  \textstyle   \regret{y} &  \textstyle := & \textstyle
     \sum_{t=1}^T  \alpha_t  \ell_t(y_t) - \min_{y} \sum_{t=1}^T  \alpha_t  \ell_t(y)\\
  \label{eq:xregret}    \textstyle  \regret{x} &  \textstyle := & \textstyle
     \sum_{t=1}^T  \alpha_t  h_t(x_t) - \sum_{t=1}^T  \alpha_t  h_t(x^*)
\end{eqnarray}
Notice that the $x$-player's regret is computed relative to $x^*$ the minimizer of $f(\cdot)$, rather than the minimizer of $\sum_{t=1}^T  \alpha_t  h_t(\cdot)$. Although slightly non-standard, this allows us to handle the unconstrained setting while Theorem~\ref{thm:convergence} still holds as desired.

At times when we want to refer to the regret on another sequence $y_1', \ldots, y_T'$ we may refer to this as $\regret{}(y_1', \ldots, y_T')$.
We also denote $A_t$ as the cumulative sum of the weights $A_t:=\sum_{s=1}^t \alpha_s$ and the weighted average regret $\avgregret{} := \frac{\regret{}}{A_T}$.
Finally, for offline constrained optimization (i.e. $\min_{x \in \K} f(x)$), we let the decision space of the benchmark/comparator in the weighted regret definition to be $\XX=\K$; for offline unconstrained optimization, we let the decision space of the benchmark/comparator to be a norm ball that contains the optimum solution of the offline problem (i.e. contains $\arg\min_{x \in \reals^n} f(x)$), which means that $\XX$ of the comparator is a norm ball. We let $\YY = \reals^d$ be unconstrained.
\begin{theorem}\label{thm:convergence} \cite{ALLW18}
  Assume a $T$-length sequence $\balpha$ are given. Suppose in Algorithm~\ref{alg:game} the online learning algorithms $\alg^x$ and $\alg^y$ have the $\balpha$-weighted average regret $\avgregret{x}$ and $\avgregret{y}$ respectively. Then the output  $(\bar{x}_{T},\bar{y}_{T})$ is an $\epsilon$-equilibrium for $g(\cdot, \cdot)$, with
$    \epsilon = \avgregret{x} + \avgregret{y}.$
\end{theorem} 

\section{An Accelerated Solution to the Fenchel Game via Optimism} \label{analysis:meta}

We are going to analyze more closely the use of Algorithm~\ref{alg:game}, with the help of Theorem~\ref{thm:convergence}, to establish a fast method to compute an approximate equilibrium of the Fenchel Game. In particular, we will establish an approximation factor of $O(1/T^2)$ after $T$ iterations, and we recall that this leads to a $O(1/T^2)$ algorithm for our primary goal of solving $\min_{x \in \K} f(x)$.

\subsection{Analysis of the weighted regret of the y-player (i.e. the gradient player)}

A very natural online learning algorithm is \FTL, which always plays the point with the lowest (weighted) historical loss
\begin{eqnarray*}
  \textstyle  \FTL \quad \quad \yftl_t & \textstyle := & \textstyle \argmin_y \left\{ \sum_{s=1}^{t-1} \alpha_s \ell_s(y) \right\}.
\end{eqnarray*}
\FTL is known to not perform well against arbitrary loss functions, but for strongly convex $\ell_t(\cdot)$ one can prove an $O(\log T/T)$ regret bound in the unweighted case. For the time being, we shall focus on a slightly different algorithm that utilizes ``optimism'' in selecting the next action:
\begin{eqnarray*}
\textstyle	\OFTL \quad \quad \yof_t & \textstyle := & \textstyle \argmin_y \left\{ \alpha_t \ell_{t-1}(y) +  \sum_{s=1}^{t-1} \alpha_s \ell_s(y)\right\}.
\end{eqnarray*}
This procedure can be viewed as an optimistic variant of \FTL since the algorithm is effectively making a bet that, while $\ell_t(\cdot)$ has not yet been observed, it is likely to be quite similar to $\ell_{t-1}$. Within the online learning community, the origins of this trick go back to \cite{CJ12}, although their algorithm was described in terms of a 2-step descent method. This was later expanded by \cite{RK13} who coined the term \emph{optimistic mirror descent} (OMD), and who showed that the proposed procedure can accelerate zero-sum game dynamics when both players utilize OMD. \OFTL, defined as a ``batch'' procedure, was first presented in \cite{ALLW18} and many of the tools of the present paper follow directly from that work.

For convenience, we'll define $\delta_t(y) := \alpha_t (\ell_t(y) - \ell_{t-1}(y))$. Intuitively, the regret will be small if the functions $\delta_t$ are not too big. This is formalized in the following lemma.
\begin{lemma} \label{lem:simple_bound}
	For an arbitrary sequence $\{ \alpha_t, \ell_t \}_{t=1\ldots T}$, the regret of \OFTL satisfies
$\textstyle		\regret{y}(\yof_1, \ldots, \yof_T)  \leq  \sum_{t=1}^T \delta_t(\yof_t) - \delta_t(\yftl_{t+1}) $.
\end{lemma}
\begin{proof}
	Let $L_t(y) := \sum_{s=1}^{t} \alpha_s \ell_s(y)$ and also $\tilde L_t(y) := \alpha_t \ell_{t-1}(y) +  \sum_{s=1}^{t-1} \alpha_s \ell_s(y)$.

	\begin{eqnarray*} 
\textstyle		\regret{}(\yof_{1:T}) & := & \textstyle \sum_{t=1}^T  \alpha_t  \ell_t(\yof_t) - L_T(\yftl_{T+1}) =  \textstyle \sum_{t=1}^{T}  \alpha_t  \ell_t(\yof_t) - \tilde L_{T}(\yftl_{T+1}) - \delta_T(\yftl_{T+1}) \\
		& \leq & \textstyle\sum_{t=1}^{T}  \alpha_t  \ell_t(\yof_t) - \tilde L_{T}(\yof_T) - \delta_T(\yftl_{T+1}) \\
		& \textstyle = & \textstyle  \sum_{t=1}^{T-1}  \alpha_t  \ell_t(\yof_t) - L_{T-1}(\yof_T) + \delta_T(\yof_T) - \delta_T(\yftl_{T+1})  \\
		& \textstyle \leq & \textstyle\sum_{t=1}^{T-1}  \alpha_t  \ell_t(\yof_t) - L_{T-1}(\yftl_T) + \delta_T(\yof_T) - \delta_T(\yftl_{T+1})\\
		& \textstyle = &\textstyle \regret{}(\yof_{1:T-1}) + \delta_T(\yof_T) - \delta_T(\yftl_{T+1}).
	\end{eqnarray*}
	The bound follows by induction on $T$. 
\end{proof}


The result from Lemma~\ref{lem:simple_bound} is generic, and would hold for any online learning problem. But for the Fenchel game, we have a very specific sequence of loss functions, $\ell_t(y) := -g(x_t, y) = f^*(y) - \langle x_t, y \rangle$. With this in mind, let us further analyze the regret of the $y$ player.

For the time being, let us assume that the sequence of $x_t$'s is arbitrary. We define
\begin{eqnarray*}
	\xav_t := \textstyle \frac{1}{A_t}\sum_{s=1}^t \alpha_s x_s \quad \quad \text{ and } \quad \quad \xof_t := \textstyle \frac{1}{A_t}(\alpha_t x_{t-1} + \sum_{s=1}^{t-1} \alpha_s x_s ).
\end{eqnarray*}
It is critical that we have two parallel sequences of iterate averages for the $x$-player. Our final algorithm will output $\xav_T$, whereas the Fenchel game dynamics will involve computing $\nabla f$ at the \emph{reweighted averages} $\xof_t$ for each $t=1, \ldots, T$.

To prove the key regret bound for the $y$-player, we first need to state some simple technical facts.
\begin{eqnarray}
	\label{eq:haty}
	\yftl_{t+1} & = & \argmin_y \sum_{s=1}^t \alpha_s \left( f^*(y) - \langle x_s, y \rangle\right) = \argmax_y  \left \langle \xav_t, y \right \rangle - f^*(y) = \nabla f(\xav_{t}) \\
	\label{eq:tildey} \yof_t & = & \nabla f(\xof_t)  \quad \quad \quad \quad \quad \quad \quad \quad \text{(following same reasoning as above)}, \\
	\label{eq:xdiffs} \xof_t - \xav_t & = & \frac{\alpha_t}{A_t}(x_{t-1} - x_{t}).
\end{eqnarray}
Equations~\ref{eq:haty} and~\ref{eq:tildey} follow from elementary properties of Fenchel conjugation and the Legendre transform \cite{R96}. Equation~\ref{eq:xdiffs} follows from a simple algebraic calculation.

\begin{lemma} \label{lem:yregretbound}
Suppose $f(\cdot)$ is a convex function that is $L$-smooth with respect to the the norm $\| \cdot \|$ with dual norm $\| \cdot \|_*$. Let $x_1, \ldots, x_T$ be an arbitrary sequence of points. Then, we have
\begin{equation} \label{wregret_y}
\textstyle
\regret{y}(\yof_1, \ldots, \yof_T) \leq L \sum_{t=1}^T \frac{\alpha_t^2}{A_t} \|x_{t-1} - x_t \|^2.
\end{equation}
\end{lemma}

\begin{proof}
 Following Lemma~\ref{lem:simple_bound}, and noting that here we have $\delta_t(y) = \alpha_t \langle x_{t-1} - x_{t}, y \rangle$, we have

    \begin{eqnarray*}
         \textstyle \sum_{t=1}^{T} \alpha_t \ell_t(\yof_t) - \alpha_t \ell_t(y^*) 
        & \leq &  \textstyle \sum_{t=1}^T \delta_t(\yof_t) - \delta_t(\yftl_{t+1}) = \textstyle  \sum_{t=1}^T \alpha_t \langle x_{t-1} - x_{t}, \yof_t - \yftl_{t+1} \rangle \\
        \text{(Eqns. \ref{eq:haty}, \ref{eq:tildey})} \quad \quad
        & = & \textstyle  \sum_{t=1}^T \alpha_t \langle x_{t-1} - x_{t}, \nabla f(\xof_t) - \nabla f(\xav_t) \rangle \\
        \text{(H\"older's Ineq.)} \quad \quad
        & \leq & \textstyle  \sum_{t=1}^T \alpha_t \| x_{t-1} - x_{t}\| \| \nabla f(\xof_t) - \nabla f(\xav_t) \|_* \\
        \text{($L$-smoothness of $f$)} \quad \quad
        & \leq & \textstyle   L \sum_{t=1}^T \alpha_t \| x_{t-1} - x_{t}\| \|\xof_t - \xav_t \| \\
        \text{(Eqn. \ref{eq:xdiffs})} \quad \quad
        & = & \textstyle  L \sum_{t=1}^T \frac{\alpha_t^2}{A_t} \| x_{t-1} - x_{t}\| \|x_{t-1} - x_{t} \|
    \end{eqnarray*}
as desired.
\end{proof}

We notice that a similar bound is given in \cite{ALLW18} for the gradient player using 
\OFTL, yet the above result is a stict improvement as the previous work relied on the additional assumption that $f(\cdot)$ is strongly convex. The above lemma depends only on the fact that $f$ has lipschitz gradients.

\subsection{Analysis of the weighted regret of the x-player}

In the present section we are going to consider that the $x$-player uses 
\MD for updating its action, which is defined as follows. 
\begin{equation} \label{MDupdate}
\textstyle  x_t  :=  \textstyle \argmin_{x \in \K}   \alpha_t h_t(x) +  \frac{1}{\gamma_t} V_{x_{t-1}}(x) 
   =  \argmin_{x \in \K}   \gamma_t \langle x, \alpha_t y_t \rangle  + V_{x_{t-1}}(x),
\end{equation}
where we recall that the Bregman divergence $V_x(\cdot)$ is with respect to a 1-strongly convex regularization $\phi$. Also, we note that the $x$-player has an advantage in these game dynamics, since $x_t$ is chosen \emph{with knowledge of} $y_t$ and hence has knowledge of the incoming loss $h_t(\cdot)$.

\begin{lemma}\label{lem:xregretbound}
  Let the sequence of $x_t$'s be chosen according to \MD. Assume that the Bregman Divergence is uniformly bounded on $\K$, so that $D = \sup_{t=1, \ldots, T} V_{x_t}(x^*)$, where $x^{*}$ denotes the minimizer of $f(\cdot)$.
   Assume that the sequence $\{\gamma_t\}_{t=1,2,\ldots}$ is non-increasing. Then we have
$    \regret{x} \leq  \frac D {\gamma_{T}} - \sum_{t=1}^{T} \frac{1}{2 \gamma_t} \| x_{t-1} - x_t \|^2.$
\end{lemma}
The proof of this lemma is quite standard, and we postpone it to Appendix~\ref{app:mirror}.
We also note that the benchmark $x^{*}$ is always within a finite norm ball by assumption. We given an alternative to this lemma in the appendix, when $\gamma_t = \gamma$ is fixed, in which case we can instead use the more natural constant $D = V_{x_1}(x^*)$.

\subsection{Convergence Rate of the Fenchel Game}


\begin{theorem}\label{thm:meta}
Let us consider the output $(\xav_T, \yav_T)$ of Algorithm~\ref{alg:game} under the following conditions: (a) the sequence $\{\alpha_t\}$ is positive but otherwise arbitrary (b) $\alg^y$ is chosen \OFTL, (c) $\alg^x$ is \MD with any non-increasing positive sequence $\{\gamma_t\}$, and (d) we have a bound $V_{x_t}(x^*) \leq D$ for all $t$. Then the point $\xav_T$ satisfies
\begin{equation} \label{eq:genericbound}
\displaystyle  f(\xav_T) - \min_{x \in \XX} f(x)
  \leq \frac 1 {A_T} \left( \frac D {\gamma_{T}} 
    + \sum_{t=1}^{T} \left(\frac{\alpha_t^2}{A_t} L - \frac{1}{2 \gamma_t} \right)
         \| x_{t-1} - x_t \|^2    \right).
\end{equation}
\end{theorem}
\begin{proof}
  We have already done the hard work to prove this theorem. Lemma~\ref{lem:fenchelgame} tells us we can bound the error of $\xav_T$ by the $\epsilon$ error of the approximate equilibrium $(\xav_T, \yav_T)$. Theorem~\ref{thm:convergence} tells us that the pair $(\xav_T, \yav_T)$ derived from Algorithm~\ref{alg:game} is controlled by the sum of averaged regrets of both players, $\frac{1}{A_T}(\regret{x} + \regret{y})$. But we now have control over both of these two regret quantities, from Lemmas~\ref{lem:yregretbound} and~\ref{lem:xregretbound}. The right hand side of \eqref{eq:genericbound} is the sum of these bounds.
\end{proof}

Theorem~\ref{thm:meta} is somewhat opaque without a specifying the sequence $\{\alpha_t\}$. But what we now show is that the summation term \emph{vanishes} when we can guarantee that $\frac{\alpha_t^2}{A_t}$ remains constant! This is where we obtain the following fast rate.
\begin{corollary} \label{cor:meta}
  Following Theorem~\ref{thm:meta} with $\alpha_t = t$ and for any non-increasing sequence $\gamma_t$ satisfying $\frac{1}{CL} \leq \gamma_t \leq \frac{1}{4L}$ for some constant $C > 4$, we have
$   \displaystyle f(\xav_T) - \min_{x \in \XX} f(x) 
  \leq \frac{2 C L D} {T^2}. 
$
\end{corollary}
\begin{proof}
  Observing $A_t := \frac{t(t+1)}{2}$, the choice of $\{\alpha_t,\gamma_t\}$ implies $\frac{D}{\gamma_t} \leq cDL$ and $\frac{L\alpha_t^2}{A_t} = \frac{2Lt^2}{t(t+1)} \leq 2L \leq \frac 1 {2 \gamma_t}$, which ensures that the summation term in \eqref{eq:genericbound} is negative. The rest is simple algebra.
\end{proof}
A straightforward choice for the learning rate $\gamma_t$ is simple the constant sequence $\gamma_t = \frac 1 {4L}$. The corollary is stated with a changing $\gamma_t$ in order to bring out a connection to the classical \NA in the following section.

\textbf{Remark:} It is worth dwelling on exactly how we obtained the above result. A less refined analysis of the \MD algorithm would have simply ignored the negative summation term in Lemma~\ref{lem:xregretbound}, and simply upper bounded this by 0. But the negative terms $\| x_t - x_{t-1} \|^2$ in this sum happen to correspond \textit{exactly} to the positive terms one obtains in the regret bound for the $y$-player, but this is true \textit{only as a result of} using the \OFTL algorithm. To obtain a cancellation of these terms, we need a $\gamma_t$ which is roughly constant, and hence we need to ensure that $\frac{\alpha_t^2}{A_t} = O(1)$. The final bound, of course, is determined by the inverse quantity $\frac 1 {A_T}$, and a quick inspection reveals that the best choice of $\alpha_t = \theta(t)$. This is not the only choice that could work, and we conjecture that there are scenarios in which better bounds are achievable for different $\alpha_t$ tuning. We show in Section~\ref{sec:linearrate} that a \emph{linear rate} is achievable when $f(\cdot)$ is also strongly convex, and there we tune $\alpha_t$ to grow exponentially in $t$ rather than linearly.

\section{Nesterov’s methods are instances of our accelerated solution to the game}

Starting from 1983, Nesterov has proposed three accelerated methods for smooth convex problems
(i.e. \cite{N83a,N83b,N88,N05}. In this section, we show that our accelerated algorithm to the \textit{Fenchel game} can generate all his methods with some simple tweaks.

\subsection{Recovering Nesterov’s (1983) method for unconstrained smooth convex problems \cite{N83a,N83b}}

In this subsection, we assume that the x-player’s action space is unconstrained.
That is, $\K= \reals^n$. 
Consider the following algorithm.

\begin{algorithm} 
   \caption{A variant of our accelerated algorithm.} \label{alg:AccHeavy}
\begin{algorithmic}[1]
\STATE In the weighted loss setting of Algorithm~\ref{alg:game}:
\STATE \quad $y$-player uses \textsc{OptimisitcFTL} as $\alg^y$: $y_t = \nabla f( \xof_{t})$.
\STATE \quad $x$-player uses \textsc{OnlineGradientDescent} as $\alg^x$:
\STATE \qquad $ x_t = x_{t-1} - \gamma_t \alpha_t \nabla h_t(x) = x_{t-1} - \gamma_t \alpha_t y_t = x_{t-1} - \gamma_t \alpha_t \nabla f( \xof_{t} ).$
\end{algorithmic}
\end{algorithm}

\begin{theorem}\label{thm:accHeavy} 
Let $\alpha_{t}=t$. 
Assume $\K = \reals^{n}$.  
Algorithm~\ref{alg:AccHeavy} is actually the case
the x-player uses \textsc{MirrorDescent}. 
Therefore, $\bar{x}_{T}$ is an $O(\frac{1}{T^2})$-approximate optimal solution of $\min_{x} f(x)$
by Theorem~\ref{thm:meta} and Corollary~\ref{cor:meta}.
\end{theorem} 

\begin{proof}
For the unconstrained case, we can let the distance generating function of the Bregman divergence to be the squared of L2 norm, i.e.
$\phi(x):=\frac{1}{2}\| x \|^{2}_2$.
Then, the update becomes $x_{t} = \argmin_{x}   \gamma_t \langle x, \alpha_t y_t \rangle  + V_{x_{t-1}}(x) = \argmin_{x}  \gamma_t \langle x, \alpha_t y_t \rangle  + \frac{1}{2} \| x \|^{2}_{2} - \langle x_{t-1}, x - x_{t-1} \rangle - \frac{1}{2} \| x_{t-1} \|^{2}_{2} $. 
Differentiating the objective w.r.t $x$ and setting it to zero, one will get
$x_{t} = x_{{t-1}} - \gamma_t \alpha_{t} y_{t}$.
\end{proof}

Having shown that Algorithm~\ref{alg:AccHeavy} is actually our accelerated algorithm to the
\textit{Fenchel game}.
We are going to show that Algorithm~\ref{alg:AccHeavy} 
has a direct correspondence with Nesterov’s first acceleration method (Algorithm~\ref{alg:Nes_v1}) \cite{N83a,N83b} (see also \cite{SBC14}).
\begin{algorithm} 
   \caption{Nesterov Algorithm [\cite{N83a,N83b}]} \label{alg:Nes_v1}
\begin{algorithmic}[1]
\STATE  Init: $w_{0}=z_{0}$. Require: $\theta \leq \frac{1}{L}$.
\FOR{$t= 1, 2, \dots, T$}
\STATE  $w_t = z_{t-1} - \theta \nabla f(z_{t-1})$.
\STATE  $z_t = w_t + \frac{t-1}{t+2} (w_t - w_{t-1}) $.
\ENDFOR
\STATE Output $w_T$.
\end{algorithmic}
\end{algorithm}

To see the equivalence, let us re-write $\bar{x}_t := \frac{1}{A_t} \sum_{s=1}^t \alpha_{s} x_{s}$ of Algorithm~\ref{alg:AccHeavy}.
\begin{equation} \label{acc_expand}
\begin{aligned}
& \textstyle \bar{x}_t = \frac{A_{t-1} \bar{x}_{t-1} + \alpha_t x_t    }{A_t}
            = \frac{A_{t-1} \bar{x}_{t-1} + \alpha_t ( x_{t-1} - \gamma_t \alpha_t \nabla f( \xof_t) ) }{A_t}
\\ & \textstyle             = \frac{A_{t-1} \bar{x}_{t-1} + \alpha_t ( \frac{ A_{t-1}\bar{x}_{t-1} - A_{t-2}\bar{x}_{t-2}  }{ \alpha_{t-1} } - \gamma_t \alpha_t \nabla f( \xof_t) )}{A_t}
\\ &   \textstyle          = \bar{x}_{t-1} ( \frac{A_{t-1}}{A_t} + \frac{\alpha_t (\alpha_{t-1} + A_{t-2} ) }{ A_t \alpha_{t-1} }   ) - \bar{x}_{t-2} ( \frac{\alpha_t A_{t-2} }{ A_t \alpha_{t-1} }   ) - \frac{\gamma_t \alpha_t^2}{A_t} \nabla f( \xof_t )
\\ &  \textstyle           = \bar{x}_{t-1} - \frac{\gamma_t \alpha_t^2}{A_t} \nabla f(\xof_t) + ( \frac{\alpha_t A_{t-2} }{ A_t \alpha_{t-1} }   ) (\bar{x}_{t-1} -  \bar{x}_{t-2} ) 
\\&  \textstyle = \bar{x}_{t-1} - \frac{1}{4L} \nabla f(\xof_t) + ( \frac{t-2 }{ t+1 }   ) (\bar{x}_{t-1} -  \bar{x}_{t-2} ) ,
\end{aligned}
\end{equation}
where $\alpha_{t}=t$ and $\gamma_{t}= \frac{(t+1)}{t} \frac{1}{8L}$.

\begin{theorem}\label{thm:Nes1988} 
Algorithm~\ref{alg:Nes_v1} with $\theta=\frac{1}{4L}$ is equivalent to Algorithm~\ref{alg:AccHeavy} with  $\gamma_t = \frac{(t+1)}{t} \frac{1}{8L}$ in the sense that they generate equivalent sequences of iterates:
\[
\textnormal{for all } t=1,2,\ldots,T, \quad \quad \quad \quad w_t = \xav_t \quad \text{ and } \quad z_{t-1} = \xof_t.
\]
\end{theorem}

Let us switch to comparing the update of Algorithm~\ref{alg:AccHeavy}, which is (\ref{acc_expand}),
with the update of the \HB algorithm.
We see that (\ref{acc_expand}) has the so called momentum term (i.e. has a $(\bar{x}_{t-1} -  \bar{x}_{t-2}$) term). But, the difference is that the gradient is evaluated at $\xof_{t} = \frac{1}{A_t}(\alpha_t x_{t-1} + \sum_{s=1}^{t-1} \alpha_s x_s )$, not 
$\bar{x}_{t-1} = \frac{1}{A_{t-1}} \sum_{s=1}^{t-1} \alpha_s x_s$, which is the consequence that the y-player plays \OFTL.
To elaborate, let us consider a scenario (shown in Algorithm~\ref{alg:Heavy}) such that the $y$-player plays \FTL instead of \OFTL.

\begin{algorithm} 
   \caption{ \HB algorithm } \label{alg:Heavy}
\begin{algorithmic}[1]
\STATE In the weighted loss setting of Algorithm~\ref{alg:game}:
\STATE \quad $y$-player uses \FTL as $\alg^y$: $y_t = \nabla f( \xav_{t-1})$.
\STATE \quad $x$-player uses \textsc{OnlineGradientDescent} as $\alg^x$:
\STATE \qquad $ x_t := x_{t-1} - \gamma_t \alpha_t \nabla h_t(x) = x_{t-1} - \gamma_t \alpha_t y_t = x_{t-1} - \gamma_t \alpha_t \nabla f( \xav_{t-1} ).$
\end{algorithmic}
\end{algorithm}

Following what we did in (\ref{acc_expand}), we can rewrite $\bar{x}_{t}$ of Algorithm~\ref{alg:Heavy} as
\begin{equation} \label{eq:hvy}
\textstyle \bar{x}_t =\bar{x}_{t-1} - \frac{\gamma_t \alpha_t^2}{A_t} \nabla f(\bar{x}_{t-1}) + (\bar{x}_{t-1} -  \bar{x}_{t-2} ) ( \frac{\alpha_t A_{t-2} }{ A_t \alpha_{t-1} }   ),
\end{equation}
by observing that (\ref{acc_expand}) still holds except that $\nabla f( \xof_t )$ is changed to $\nabla f(\bar{x}_{t-1})$ as the y-player uses \FTL now,
which give us the update of the Heavy Ball algorithm as (\ref{eq:hvy}). Moreover, by the regret analysis,
we have the following theorem. The proof is in Appendix~\ref{app:thm:Heavy}.
\begin{theorem}\label{thm:Heavy} 
Let $\alpha_{t}=t$. Assume $\K = \reals^{n}$. Also, let $\gamma_t =O(\frac{1}{L})$.
The output $\bar{x}_{T}$ of Algorithm~\ref{alg:Heavy} is an $O(\frac{1}{T})$-approximate optimal solution of $\min_{x} f(x)$.
\end{theorem}

To conclude, by comparing Algorithm~\ref{alg:AccHeavy} and Algorithm~\ref{alg:Heavy},
we see that Nesterov‘s (1983) method enjoys $O(1/T^2)$ rate since its adopts $\OFTL$, while the \HB algorithm which adopts $\textsc{FTL}$
may not enjoy the fast rate, as the distance terms may not cancel out. The result also conforms to empirical studies that the \HB does not exhibit acceleration on general smooth convex problems.

\subsection{Recovering Nesterov’s (1988) 1-memory method \cite{N88} and Nesterov’s (2005) $\infty$-memory method  \cite{N05}}

In this subsection, we consider recovering Nesterov’s (1988) 1-memory method  \cite{N88}
and Nesterov’s (2005) $\infty$-memory method \cite{N05}.
To be specific, we adopt the presentation of Nesterov’s algorithm given in Algorithm~1 
and Algorithm~3 of \cite{T08} respectively.

\begin{algorithm}[h] 
   \caption{ (A) Nesterov's 1-memory method \cite{N88} and (B) Nesterov's $\infty$-memory method \cite{N05}  } \label{alg:Nes}
\begin{algorithmic}[1]
\STATE Input: parameter $\beta_t = \frac{2}{t+1}$, $\gamma'_{t}= \frac{t}{4L}$, $\theta_t = t$, and $\eta=\frac{1}{4L}$. 
\STATE Init: $w_{0}=x_{0}$
\FOR{$t= 1, 2, \dots, T$}
\STATE $z_{t} = (1 - \beta_t) w_{t-1} + \beta_t x_{t-1}$.
\STATE (A) $x_{t} = \argmin_{x \in \K} \gamma'_t  \langle  \nabla f( z_t), x \rangle  + V_{x_{t-1}}(x)$.   
\STATE Or, (B) $x_t = \argmin_{x \in \K} \sum_{s=1}^{t} \theta_s \langle x , \nabla f(z_s) \rangle + \frac{1}{\eta} R(x),$ where $R(\cdot)$ is 1-strongly convex.
\STATE $w_{t} = (1 - \beta_t) w_{t-1} + \beta_t x_{t}$.
\ENDFOR
\STATE Output $w_{T}$.
\end{algorithmic}
\end{algorithm}

\begin{theorem}\label{thm:Nes_one} 
Let $\alpha_{t}=t$. 
Algorithm~\ref{alg:Nes} with update by option (A) is the case when the y-player uses \OFTL
and the x-player adopts \MD with $\gamma_{t} = \frac{1}{4L}$ in \textit{Fenchel game}. Therefore, $w_{T}$ is an $O(\frac{1}{T^2})$-approximate optimal solution of $\min_{{x \in \K}} f(x)$.
\end{theorem} 
The proof is in Appendix~\ref{app:thm:Nes_one}, which shows the direct correspondence of Algorithm~\ref{alg:Nes} using option (A) to our accelerated solution in Section~\ref{analysis:meta}.
\begin{theorem}\label{thm:Nes_infty} 
Let $\alpha_{t}=t$. 
Algorithm~\ref{alg:Nes} with update by option (B) is the case when the y-player uses \OFTL
and the x-player adopts \BTRL with $\eta = \frac{1}{4L}$ in \textit{Fenchel game}.
Therefore, $w_{T}$ is an $O(\frac{1}{T^2})$-approximate optimal solution of $\min_{{x \in \K}} f(x)$.
\end{theorem} 
The proof is in Appendix~\ref{app:thm:Nes_infty}, which requires the regret bound of \BTRL.

\subsection{Accelerated linear rate} \label{sec:linearrate}

Nesterov observed that, when $f(\cdot)$ is both $\mu$-strongly convex and $L$-smooth, one can achieve a rate that is exponentially decaying in $T$ (e.g. page 71-81 of \cite{N04}). It is natural to ask if the zero-sum game and regret analysis in the present work
also recovers this faster rate in the same fashion. We answer this in the affirmative.
Denote $\kappa := \frac{L}{\mu}$.
A property of $f(x)$ being $\mu$-strongly convex is that
the function
$\tilde{f}(x) := f(x) - \frac{\mu \| x\|^2_2}{2} $
is still a convex function.
Now we define a new game whose payoff function is
$\textstyle \tilde g(x,y):= \langle x, y \rangle - \tilde{f}^*(y) +\frac{\mu  \| x\|^2_2}{2}$.
Then, the minimax vale of the game is
$V^* := \min_x \max_y \tilde g(x, y) = \min_x \tilde{f}(x) + \frac{\mu  \| x\|^2_2}{2} = \min_x f(x)$.
Observe that, in this game, the loss of the y-player in round $t$ is $\alpha_t \ell_t(y): =   \alpha_t( \tilde{f}^*(y)  -\langle x_t, y \rangle )$,
while the loss of the x-player in round $t$ is a strongly convex function 
$\alpha_t h_t(y):= \alpha_t ( \langle x, y_t \rangle + \frac{\mu\|x\|^{2}_{2}}{2} ) $.
The cumulative loss function of the x-player becomes more and more strongly convex over time,
which is the key to allowing the exponential growth of the total weight $A_{t}$ that leads to the linear rate.
In this setup, we have a ``warmup round'' $t=0$, and thus we denote 
$\tilde{A}_{t}:= \sum_{s=0}^{t} \alpha_{s}$ which incorporate the additional step into the average. The proof of the following result is in Appendix~\ref{app:acc_linear}.

\begin{theorem} \label{thm:acc_linear}
For the game
$\tilde g(x,y):= \langle x, y \rangle - \tilde{f}^*(y) + \frac{\mu  \| x\|^2_2}{2}$
, if the y-player plays \textsc{OptimisticFTL}
and the x-player plays \textsc{BeTheRegularizedLeader}:
$x_t \leftarrow \arg\min_{{x \in \XX}} \sum_{s=0}^t \alpha_{s} \ell_{s}(x)$,
where $\alpha_0 \ell_{0}(x):= \alpha_{0} \frac{\mu  \| x\|^2_2}{2}$, then 
the weighted average points $(\xav_T, \yav_T)$ would be an $O(\exp(-\frac{T}{\sqrt{\kappa}}))$-approximate equilibrium
of the game, where the weights $\alpha_0, \alpha_1, \ldots$ are chosen to satisfy $\frac{\alpha_t}{\tilde{A_t} } = \frac{1}{\sqrt{6 \kappa}}$.
This implies that $f(\xav_T) - \min_{x \in \XX} f(x) = O(\exp(-\frac{T}{\sqrt{\kappa}})).$
\end{theorem}

\section{Accelerated Proximal Method} 
In this section, we consider solving composite optimization problems
$\textstyle \min_{x \in \reals^n} f(x) + \psi(x),$
where $f(\cdot)$ is smooth convex but $\psi(\cdot)$ is possibly non-differentiable convex (e.g. $\|\cdot\|_1$).
We want to show that the game analysis still applies to this problem.
We just need to change the payoff function $g$ to account for $\psi(x)$.
Specifically, we consider the following two-players zero-sum game,
$\textstyle \min_{x} \max_{y} \{ \langle x, y \rangle - f^*(y) + \psi(x)  \}.$
Notice that the minimax value of the game is $\min_{x} f(x) + \psi(x)$, which is exactly the  optimum value of the composite optimization problem.
Let us denote the proximal operator as 
$\textbf{prox}_{\lambda \psi} (v) = \argmin_x \big( \psi(x)+ \frac{1}{2\lambda} \| x - v \|^2_2 \big).$
\footnote{It is known that for some $\psi(\cdot)$, their corresponding proximal operations have closed-form solutions
(see e.g. \cite{PB14} for details).}

\begin{algorithm}[H] 
   \caption{Accelerated Proximal Method} \label{alg:AccProx}
\begin{algorithmic}[1]
\STATE In the weighted loss setting of Algorithm~\ref{alg:game} (let $\alpha_t = t$ and $\gamma_t = \frac{1}{4L}$):
\STATE \quad $y$-player uses \textsc{OptimisitcFTL} as $\alg^y$: $y_t = \nabla f( \xof_{t})$.
\STATE \quad $x$-player uses \MD  with $\psi(x):= \frac{1}{2} \|x \|^2_2$ in Bregman divergence
as $\alg^x$: \\ \STATE \qquad 
$x_{t} = \argmin_{x}   \gamma_t ( \alpha_t h_t(x) ) + V_{x_{t-1}}(x)
= \argmin_{x}   \gamma_t ( \alpha_t \{  \langle x , y_t \rangle + \psi(x)   \} ) + V_{x_{t-1}}(x)$
\\ \STATE \qquad
$ = \argmin_{x} \phi(x) + \frac{1}{2 \alpha_t \gamma_t} ( \| x \|^2_2 + 2 \langle \alpha_t \gamma_t y_t - x_{t-1} , x \rangle )  = \textbf{prox}_{\alpha_t \gamma_t \psi} ( x_{t-1}- \alpha_t \gamma_t \nabla f( \xof_{t}) ) $
\end{algorithmic}
\end{algorithm}

We notice that the loss function of the x-player here, $\alpha_t h_{t}(x) = \alpha_t ( \langle x , y_t \rangle + \psi(x) )$, is possibly nonlinear.
Yet, we can slightly adapt the analysis in Section~\ref{analysis:meta} to 
show that the weighed average $\bar{x}_{T}$ is still an $O(1/T^2)$ approximate optimal solution of the offline problem. Please see Appendix~\ref{app:accProx} for details.
One can view Algorithm~\ref{alg:AccProx} as a variant of the so called ``Accelerated Proximal Gradient''in \cite{BT09}.
Yet, the design and analysis of our algorithm is simpler than that of \cite{BT09}.

\textbf{Acknowlegement:} We would like to thank Kevin Lai and Kfir Levy for helpful discussions leading up to the results in this paper. This work was supported by funding from the Division of Computer Science and Engineering at the University of Michigan, from the College of Computing at the Georgia Institute of Technology, NSF TRIPODS award 1740776, and NSF CAREER award 1453304.

\bibliographystyle{plain}
\bibliography{nips_2018_final}  

\appendix

\section{Two key lemmas}
 \label{app:mirror}

\textbf{Lemma~\ref{lem:xregretbound}} 
\textit{
  Let the sequence of $x_t$'s be chosen according to \MD. Assume that the Bregman Divergence is uniformly bounded on $\K$, so that $D = \sup_{t=1, \ldots, T} V_{x_t}(x^*)$, where $x^{*}$ denotes the minimizer of $f(\cdot)$.
   Assume that the sequence $\{\gamma_t\}_{t=1,2,\ldots}$ is non-increasing. Then we have
$    \regret{x} \leq  \frac D {\gamma_{T}} - \sum_{t=1}^{T} \frac{1}{2 \gamma_t} \| x_{t-1} - x_t \|^2.$
}

 \begin{proof}
The key inequality we need, which can be found in Lemma 1 of \cite{RS13} (and for completeness is included in Appendix~\ref{app:mirror}) is as follows: let $y, c$ be arbitrary, and assume
$x^+ =  \argmin_{x \in \K} \langle x, y \rangle + V_{c}(x)$,  
then for any $x^* \in \K$,
$\textstyle \langle x^+ - x^*, y \rangle \leq V_{c}(x^*) - V_{x^+}(x^*) - V_{c}(x^+).$
Now apply this fact for $x^+ = x_t$, $y=\gamma_t \alpha_t y_t$ and $c = x_{{t-1}}$,
which provides
\begin{equation} \label{tta1}
\begin{aligned}
\textstyle \langle x_t - x^*,  \gamma_t \alpha_t y_t \rangle \leq V_{x_{t-1}}(x^*) - V_{x_t}(x^*) - V_{x_{t-1}}(x_t).
\end{aligned}
\end{equation} 
So, the weighted regret of the x-player can be bounded by
\begin{equation} \label{wregret_x}
\begin{aligned}
& \textstyle \regret{x} :=  
\sum_{t=1}^{T} \alpha_t \langle x_t - x^*, y_t \rangle
\overset{(\ref{tta1})}{\leq}  \sum_{t=1}^{T} \frac{1}{\gamma_t} \big( V_{x_{t-1}}(x^*) - V_{x_t}(x^*) - V_{x_{t-1}}(x_t) \big) 
\\ & \textstyle
=  \frac{1}{\gamma_1} V_{x_{0}}(x^*) - \frac{1}{\gamma_T} v_{x_T}(x^*)  + \sum_{t=1}^{T-1} ( \frac{1}{\gamma_{t+1}} - \frac{1}{\gamma_t} ) V_{x_t}(x^*) - \frac{1}{\gamma_t} V_{x_{t-1}}(x_t) 
\\ & \textstyle
\overset{(a)}{ \leq} \frac{1}{\gamma_1} D + \sum_{t=1}^{T-1} ( \frac{1}{\gamma_{t+1}} - \frac{1}{\gamma_t} ) D - \frac{1}{\gamma_t} V_{x_{t-1}}(x_t) 
= \frac{D}{\gamma_{T}} -  \sum_{t=1}^{T} \frac{1}{\gamma_t} V_{x_{t-1}}(x_t) 
\\ & \textstyle \overset{(b)}{ \leq} \frac{D}{\gamma_{T}} - \sum_{t=1}^{T} \frac{1}{2 \gamma_t} \| x_{t-1} - x_t \|^2,
\end{aligned}
\end{equation}
where $(a)$ holds since the sequence $\{\gamma_t\}$ is non-increasing and $D$ upper bounds the divergence terms, and $(b)$ follows from the strong convexity of $\phi$, which grants $V_{x_{t-1}}(x_t) \geq \frac 1 2 \| x_t - x_{t-1}\|^2$.
\end{proof}

The above lemma requires a bound $D$ on the divergence terms $V_{x_t}(x^*)$, which might be large in certain unconstrained settings -- recall that we do no necessarily require that $\K$ is a bounded set, we only assume that $f(\cdot)$ is minimized at a point with finite norm. On the other hand, when the $x$-player's learning rate $\gamma$ is fixed, we can define the more natural choice $D = V_{x_0}(x^*)$.

\textbf{Lemma~\ref{lem:xregretbound} [Alternative]:}
\textit{
  Let the sequence of $x_t$'s be chosen according to \MD, and assume $\gamma_t = \gamma$ for all $t$. Let $D = V_{x_0}(x^*)$, where $x^{*}$ denotes the benchmark in $\regret{x}$. Then we have
$    \regret{x} \leq  \frac D {\gamma} - \sum_{t=1}^{T} \frac{1}{2 \gamma} \| x_{t-1} - x_t \|^2.$
}
\begin{proof}
  The proof follows exactly as before, yet $\gamma_t = \gamma_{t+1}$ for all $t$ implies that $\frac{1}{\gamma_{t+1}} - \frac{1}{\gamma_t} = 0$ and we may drop the sum in the third line of \eqref{wregret_x}. The rest of the proof is identical.
\end{proof}

\textbf{Lemma 1 of \cite{RS13}}:
\textit{
Let $x' =  \arg\min_{x \in \K} \langle x, y \rangle + V_{c}(x)$. 
Then, it satisfies that for any $x^* \in \K$,
\begin{equation} 
\textstyle \langle x' - x^*, y \rangle \leq V_{c}(x^*) - V_{x'}(x^*) - V_{c}(x').
\end{equation}
}

\begin{proof}
Recall that the Bregman divergence with respect to the distance generating function $\phi(\cdot)$ at 
a point $c$ is: $V_{c}(x):= \phi(x) - \langle \nabla \phi(c), x - c  \rangle  - \phi(c).$ 

Denote $F(x):= \langle x, y \rangle + V_{c}(x)$.
Since $x'$ is the optimal point of 
$\arg\min_{x \in K} F(x)\textstyle,$
by optimality,
$ \langle x^* - x' ,  \nabla F(x') \rangle \geq 0$,
for any $x^{*} \in K$.
So, 
\begin{equation}
\begin{aligned}
& \textstyle \langle x^* - x' ,  \nabla F(x') \rangle
=  \langle  x^* - x', y \rangle + \langle x^* - x', \nabla \phi (x') - \nabla \phi(c) \rangle
\\  & \textstyle = \langle  x^* - x', y \rangle  + \{ \phi(x^*) - \langle \nabla \phi(c), x^* - c  \rangle  - \phi(c) \}
- \{ \phi(x^*) - \langle \nabla \phi(x'), x^* - x'  \rangle  - \phi(x') \}
\\  & \textstyle - \{ \phi(x') - \langle \nabla \phi(c), x' - c  \rangle  - \phi(c) \}
\\  & \textstyle  = \langle  x^* - x', y \rangle + V_c(x^*) - V_{x'}(x^*) -  V_c(x')
\geq 0.
\end{aligned}
\end{equation}
The last inequality means that 
\begin{equation}
\textstyle \langle x' - x^*, y \rangle \leq V_{c}(x^*) - V_{x'}(x^*) - V_{c}(x').
\end{equation}

\end{proof}

\section{Proof of Theorem~\ref{thm:Nes1988}} \label{app:thm:Nes1988}

\textbf{Theorem~\ref{thm:Nes1988}}
\textit{
Algorithm~\ref{alg:Nes_v1} with $\theta=\frac{1}{4L}$ is equivalent to Algorithm~\ref{alg:AccHeavy} with  $\gamma_t = \frac{(t+1)}{t} \frac{1}{8L}$ in the sense that they generate equivalent sequences of iterates:
\[
\textnormal{for all } t=1,2,\ldots,T, \quad \quad \quad \quad w_t = \xav_t \quad \text{ and } \quad z_{t-1} = \xof_t.
\]
}
\begin{proof}
First, let us check the base case to see if $w_{1}= \bar{x}_{1}$. We have that $w_{1} = z_0 -\theta \nabla f(z_{0})$ from line 3 of Algorithm~\ref{alg:Nes_v1},
while $\bar{x}_{1} = \bar{x}_0 - \frac{1}{4L} \nabla f(\xof_1)$ in (\ref{acc_expand}).
Thus, if the initialization is the same: $w_{0}=z_{0}=x_{0}=\bar{x}_{0} = \xof_{1}$, then $w_{1}=\bar{x}_{1}$.

Now assume that $w_{t-1} = \bar{x}_{t-1}$ holds for a $t\geq 2$.
Then, from the expression of line 4 that 
$z_{t-1} = w_{t-1} + \frac{t-2}{t+1} (w_{t-1} - w_{t-2}) $,
we get $z_{t-1} = \bar{x}_{t-1} + \frac{t-2}{t+1} (\bar{x}_{t-1} - \bar{x}_{t-2}).$
Let us analyze that the r.h.s of the equality.
The coefficient of $x_{{t-1}}$ in 
$\bar{x}_{t-1} + \frac{t-2}{t+1} (\bar{x}_{t-1} - \bar{x}_{t-2})$
is $\frac{ (t-1) + \frac{t-2}{t+1}{(t-1)}   }{A_{t-1}} = \frac{2(1+\frac{t-2}{t+1})}{t} = \frac{2 (2t-1)}{t(t+1)}$, while the coefficient of each $x_{\tau}$ for any $\tau \leq t-2$ in $\bar{x}_{t-1} + \frac{t-2}{t+1} (\bar{x}_{t-1} - \bar{x}_{t-2})$
is $\frac{ ( 1 + \frac{t-2}{t+1} ) \tau }{A_{t-1}} - \frac{t-2}{t+1} \frac{\tau}{A_{t-2}}
= \{ \frac{2(2t-1)}{(t-1)t(t+1)} - \frac{2}{(t+1)(t-1)}     \} \times \tau 
= \{ \frac{2}{(t-1)(t+1)} \big( \frac{2t-1}{t}  - 1 \big)   \} \times \tau
=  \frac{2 \tau}{t(t+1)}$.
Yet, 
the coefficient of $x_{t-1}$ in $\widetilde{x}_{t}$ is $\frac{t+(t-1)}{A_t} = \frac{2 (2t-1)}{t(t+1)}$
and the coefficient of $x_{\tau}$ in $\widetilde{x}_{t}$ is $\frac{\tau}{A_t} = \frac{2 \tau}{t(t+1)}$ for any $\tau\leq t-2$. Thus, $z_{t-1}= \widetilde{x}_{t}$.
Now observe that if $z_{t-1}= \widetilde{x}_{t}$, we get $w_{t} = \bar{x}_{t}$.
To see this,
substituting  $z_{t-1} = w_{t-1} + \frac{t-2}{t+1} (w_{t-1} - w_{t-2}) $ of line 4
into line 3,
we get 
$w_{t}  = w_{t-1} + \frac{t-2}{t+1} (w_{t-1} - w_{t-2}) - \theta \nabla f(z_{t-1})$.
By using $z_{t-1}= \widetilde{x}_{t}$ and $w_{t-1} = \bar{x}_{t-1}$, we further get $w_{t} = \bar{x}_{t-1} + \frac{t-2}{t+1} (\bar{x}_{t-1} - \bar{x}_{t-2}) - \theta \nabla f(\widetilde{x}_t) = \bar{x}_{t}$.
We can repeat the argument to show that the correspondence holds for any $t$,
which establishes the equivalency.

Notice that the choice of decreasing sequence $\{ \gamma_{t} \}$ here can still make the distance terms in 
(\ref{eq:genericbound}) cancel out. So, we get $O(1/T^2)$ rate by the guarantee.
\end{proof}

\section{Proof of Theorem~\ref{thm:Heavy}} \label{app:thm:Heavy}

\textbf{Theorem~\ref{thm:Heavy}}
\textit{
Let $\alpha_{t}=t$. Assume $\K = \reals^{n}$. Also, let $\gamma_t =O(\frac{1}{L})$.
The output $\bar{x}_{T}$ of Algorithm~\ref{alg:Heavy} is an $O(\frac{1}{T})$-approximate optimal solution of $\min_{x} f(x)$.
}
\begin{proof}

To analyze the guarantee of $\bar{x}_{T}$ of Algorithm~\ref{alg:Heavy}, we use the following lemma about \FTL for strongly convex loss functions.

\textbf{Corollary 1 from \cite{KS09}}
\textit{
Let $\ell_1, . . . ,\ell_T$ be a sequence of functions such that for all $t \in [T]$, $\ell_t$ is $\sigma_t$-strongly convex. Assume that \FTL runs on this sequence and for each $t \in [T]$, let $\theta_t$ be in $\nabla \ell_t(y_t)$. Then,
$ \textstyle
\sum_{t=1}^T \ell_t(y_t) - \min_x \sum_{t=1}^T \ell_t(y) \le \frac{1}{2}\sum_{t=1}^T \frac{ \| \theta_t\|^2}{\sum_{\tau=1}^{t} \sigma_{\tau}}
$
}

Observe that the $y$-player plays \FTL on the loss function sequence
$ \alpha_t \ell_t(y) := \alpha_t ( - \langle x_t, y \rangle + f^*(y) )$,
whose strong convexity parameter is $\frac{\alpha_t}{L}$
(due to $f^*(y)$ is $\frac{1}{L}$-strongly convex by duality).
Also,
 $ \nabla \ell_t(y_t)  = - x_t  + \nabla f^*(y_t) = - x_t  + \bar{x}_{t-1}$,
 where the last inequality is due to that if 
$  y_t = \argmax_y \langle \frac{1}{A_{t-1}}\sum_{s=1}^{t-1} \alpha_s x_s, y \rangle - f^*(y) = \nabla f(\xav_{t-1})$, then $\bar{x}_{t-1} = \nabla f^*(y_t)$ by duality.
So, we have 
$
  \textstyle \avgregret{y} \overset{Above Cor.}{\leq} \frac{1}{ 2 A_T} \sum_{t=1}^T \frac{ \alpha_t^2 \| \bar{x}_{t-1} - x_t  \|^2 }{ \sum_{\tau=1}^t \alpha_{\tau} (1/L)} = 
    \frac 1 {2 A_T} \sum_{t=1}^T \frac{ \alpha_t^2 L \| \bar{x}_{t-1} - x_t  \|^2 }{ A_t} =
    O( \sum_{\tau=1}^T \frac{ L \| \bar{x}_{t-1} - x_t  \|^2 }{A_T}).$
For the $x$-player, it is an instance of \textsc{MirrorDescent}, so  
$\textstyle \avgregret{x} := \frac{1}{A_T} \sum_{t=1}^{T} \langle x_t - x^{*},  \alpha_{t} y_{t} \rangle \leq  \frac{ \frac{1}{\gamma_T} D - \sum_{t=1}^{T}  \frac{1}{2 \gamma_t}   \| x_{t-1} - x_t \|^2 }{A_T}$
Therefore, $\bar{x}_{T}$ of Algorithm~\ref{alg:Heavy} is an $\avgregret{x} + \avgregret{y} = O(\frac{L \sum_{{t=1}}^{T} (\| \bar{x}_{t-1} - x_t  \|^2 -\| x_{t} - x_{t-1} \|^2  )}{A_T})$ -approximate optimal solution.
Since the distance terms may not cancel out, one may only bound the differences of the distance terms by a constant, which leads to the non-accelerated $O(1/T)$ rate.
\end{proof}

\section{Proof of Theorem~\ref{thm:Nes_one}}  \label{app:thm:Nes_one}

\textbf{Theorem~\ref{thm:Nes_one}}
\textit{
Let $\alpha_{t}=t$. 
Algorithm~\ref{alg:Nes} with update by option (A) is the case when the y-player uses \OFTL
and the x-player adopts \MD with $\gamma_{t} = \frac{1}{4L}$ in \textit{Fenchel game}. Therefore, $w_{T}$ is an $O(\frac{1}{T^2})$-approximate optimal solution of $\min_{{x \in \K}} f(x)$.
}

\begin{proof}

We first prove by induction showing that $w_{t}$ in Algorithm~\ref{alg:Nes} is $\sum_{{s=1}}^{t} \frac{\alpha_s}{A_t} x_{s}$
for any $t>0$.
For the base case $t=1$, we have
$ w_1    = (1-\beta_1) w_0 + \beta_1 x_1    = x_1 = \frac{\alpha_1}{A_1} x_1.$
Now suppose that the equivalency holds at $t-1$, for a $t \geq 2$.
Then,
\begin{equation} \label{Neseq_x}
\begin{aligned}
& \textstyle  w_{t}  = (1-\beta_t) w_{t-1} + \beta_t x_{t}
\overset{(a)}{ =} (1-\beta_t) ( \sum_{{s=1}}^{t-1} \frac{\alpha_s}{A_{t-1}} x_{s}   ) + \beta_t x_{t}
\\ & \textstyle  =  (1 - \frac{2}{t+1}) ( \sum_{{s=1}}^{t-1} \frac{\alpha_s}{ \frac{t(t-1)}{2} } x_{s}   ) + \beta_t x_{t} = \sum_{{s=1}}^{t-1} \frac{\alpha_s}{ \frac{t(t+1)}{2} } x_{s} +  \frac{ \alpha_t}{A_t}  x_{t}
 = \sum_{{s=1}}^{t} \frac{\alpha_s}{A_s} x_s, 
\end{aligned}
\end{equation}
where $(a)$ is by induction. So, it holds at $t$ too.
Now we are going to show that
$z_{t} = \frac{1}{A_t}(\alpha_t x_{t-1} + \sum_{s=1}^{t-1} \alpha_s x_s ) = \xof_t$.
We have that
$ \textstyle z_{t}  = (1-\beta_t) w_{t-1} + \beta_t x_{t-1}
= (1-\beta_t) ( \sum_{{s=1}}^{t-1} \frac{\alpha_s}{A_{t-1}} x_{s}   ) + \beta_t x_{t-1}
 \textstyle =  (1 - \frac{2}{t+1}) ( \sum_{{t=1}}^{t-1} \frac{\alpha_t}{ \frac{t(t-1)}{2} } x_{t}   ) + \beta_t x_{t-1} = \sum_{{s=1}}^{t-1} \frac{\alpha_s}{ \frac{t(t+1)}{2} } x_{s} +  \beta_t x_{t-1}
 = \sum_{{s=1}}^{t-1} \frac{\alpha_s}{ A_{t} } x_{s} +  \frac{\alpha_t}{A_t} x_{t-1}
 = \xof_t.
 $
The result also means that $\nabla f(z_{t}) = \nabla f(\xof_t) = y_t$
of the y-player who plays \texttt{Optimistic-FTL} in Algorithm~\ref{alg:game}.
Furthermore, it shows that line 5 of Algorithm~\ref{alg:Nes}:
$x_{t} = \argmin_{x \in \K} \gamma_t'  \langle  \nabla f( z_t), x \rangle  + V_{x_{t-1}}(x)$
is exactly (\ref{MDupdate}) of \MD in \textit{Fenchel game}.
Also, from (\ref{Neseq_x}), the last iterate $w_{T}$ in Algorithm~\ref{alg:Nes}
corresponds to the final output of our accelerated solution to \textit{Fenchel game}, 
which is the weighted average point that enjoys the guarantee by the game analysis.
\end{proof}

\section{Proof of Theorem~\ref{thm:Nes_infty}}  \label{app:thm:Nes_infty}

\textbf{Theorem~\ref{thm:Nes_infty}}
\textit{
Let $\alpha_{t}=t$. 
Algorithm~\ref{alg:Nes} with update by option (B) is the case when the y-player uses \OFTL
and the x-player adopts \BTRL with $\eta = \frac{1}{4L}$ in \textit{Fenchel game}.
Therefore, $w_{T}$ is an $O(\frac{1}{T^2})$-approximate optimal solution of $\min_{{x \in \K}} f(x)$.
}

\begin{proof}
Consider in \textit{Fenchel game} that the y-player uses \OFTL while the x-player plays according to \textsc{BTRL}:
    $$ \textstyle x_t = \argmin_{x \in \K} \sum_{t=1}^{T} \langle x_t , \alpha_t y_t \rangle + \frac{1}{\eta} R(x),$$
where $R(\cdot)$ is a 1-strongly convex function.
Define, $z = \arg \min_{x \in \K} R(x)$.
     Form \cite{ALLW18} (also see Appendix~\ref{app:BTRL}), it shows that  
    \textsc{BTRL} has regret
  \begin{equation} \label{wregret_x2}
    \textstyle \text{Regret}:= \sum_{t=1}^{T} \langle x_t - x^{*}, \alpha_t y_t \rangle 
\leq \frac{ R(x^*) - R(z)  - \frac{1}{2} \sum_{{t=1}}^{T} \| x_{t} - x_{t-1} \|^2}{\eta},
  \end{equation}
  where $x^*$ is the benchmark/comparator defined in the definition of the weighted regret (\ref{eq:xregret}).

By combining (\ref{wregret_y}) and (\ref{wregret_x2}),
we get that
\begin{equation} 
\begin{aligned}
\textstyle \frac{ \regret{x} + \regret{y}  }{A_T} 
= \frac{ \frac{ R(x^*) - R(z)}{\eta} + \sum_{t=1}^{T} (\frac{\alpha_t^2}{A_t} L - \frac{1}{2\eta} ) \| x_{t-1} - x_t \|^2  }{A_T} \leq O( \frac{L (R(x^*) - R(z) )  }{T^2} ),
\end{aligned}
\end{equation}
where the last inequality is because $\eta = \frac{1}{4L}$ so that the distance terms cancel out.
So, by Lemma~\ref{lem:fenchelgame} and Theorem~\ref{thm:convergence} again,
we know that $\bar{x}_{T}$ is an $O(\frac{1}{T^2})$-approximate optimal solution of $\min_{{x \in \K}} f(x)$.

The remaining thing to do is showing that $\bar{x}_{T}$ is actually $w_{T}$ of Algorithm~\ref{alg:Nes} with option (B). But, this follows the same line as the proof of Theorem~\ref{thm:Nes_one}.
So, we have completed the proof.
\end{proof}

\section{Proof of \BTRL’s regret} \label{app:BTRL}

For completeness, we replicate the proof in \cite{ALLW18}
about the regret bound of \BTRL in this section.

\textbf{Theorem 10} of  [\cite{ALLW18}]
\textit{
Let $\theta_{t}$ be the loss vector in round $t$. 
Let the update of \textsc{BTRL} be
$\textstyle x_t  = \arg \min_{x \in \K} \langle x , L_{t}  \rangle + \frac{1}{\eta} R(x)$,
where $R(\cdot)$ is $\beta$-strongly convex.
Denote $\textstyle z = \arg \min_{x \in \K} R(x)$.
Then, BTRL has regret
\begin{equation} \label{btrl}
\textstyle \text{Regret}:= \sum_{t=1}^{T} \langle x_t - x^{*}, \theta_{t} \rangle 
\leq \frac{ R(x^*) - R(z)  - \frac{\beta}{2} \sum_{{t=1}}^{T} \| x_{t} - x_{t-1} \|^2}{\eta}. 
\end{equation}
}

To analyze the regret of \BTRL, 
let us consider \OFTRL first.
Let $\theta_{t}$ be the loss vector in round $t$ and let the cumulative loss vector be
$L_{t} = \sum_{s=1}^t \theta_{s}$.
The update of \OFTRL is
\begin{equation} \label{optFTRL}
\textstyle x_t  = \arg \min_{x \in \K} \langle x , L_{t-1} + m_t \rangle + \frac{1}{\eta} R(x),
\end{equation}
where 
$m_{t}$ is the learner's guess of the loss vector in round $t$,
$R(\cdot)$ is $\beta$-strong convex with respect to a norm ($\| \cdot \|$) and $\eta$ is a parameter. 
Therefore, it is clear that
the regret of \BTRL will be the one when \OFTRL’s guess of the loss vectors exactly match the true ones, i.e. $m_{t} = \theta_{t}$.

\textbf{Theorem 16} of  [\cite{ALLW18}]
\textit{
Let $\theta_{t}$ be the loss vector in round $t$. 
Let the update of \OFTRL be
$\textstyle x_t  = \arg \min_{x \in \K} \langle x , L_{t-1} + m_t \rangle + \frac{1}{\eta} R(x)$,
where $m_{t}$ is the learner’s guess of the loss vector in round $t$ and $R(x)$ is a $\beta$-strongly 
convex function.
Denote the update of standard \textsc{FTRL} as
$\textstyle z_t  = \arg \min_{x \in \K} \langle x , L_{t-1} \rangle + \frac{1}{\eta} R(x)$.
Also, $\textstyle z_{1} = \arg \min_{x \in \K}  R(x)  $.
Then,
\OFTRL (\ref{optFTRL}) has regret
\begin{equation} 
\textstyle \text{Regret}:= \sum_{t=1}^{T} \langle x_t - x^{*}, \theta_{t} \rangle 
\leq \frac{ R(x^*) - R(z_1) - D_T }{\eta} +
\sum_{t=1}^{T} \frac{\eta}{\beta} \| \theta_t - m_{t} \|^2_* ,
\end{equation}
where $D_{T} = \sum_{{t=1}}^{T} \frac{\beta}{2} \| x_t - z_{t} \|^2 
+ \frac{\beta}{2} \| x_t - z_{t+1} \|^2$,
$z_{t} = \argmin_{x \in \K} \langle x, L_{{t-1}} \rangle + \frac{1}{\eta} R(x)$,
and $x_{t} = \arg \min_{x \in \K} \langle x , L_{t-1} + m_t \rangle + \frac{1}{\eta} R(x)$.
}

Recall that the update of \BTRL is 
$\textstyle x_t  = \arg \min_{x \in \K} \langle x , L_{t}  \rangle + \frac{1}{\eta} R(x)$,
Therefore, we have that $m_t = \theta_t$ and $x_{t}= z_{{t+1}}$ in the regret bound of \OFTRL indicated by the theorem.
Consequently, we get that 
the regret of \BTRL satisfies
\begin{equation}
 \textstyle \text{Regret}:= \sum_{t=1}^{T} \langle x_t - x^{*}, \theta_{t} \rangle 
\leq \frac{ R(x^*) - R(z) - \frac{\beta}{2} \sum_{{t=1}}^{T} \| x_{t} - x_{t-1} \|^2}{\eta}. 
\end{equation}

\section{Proof of \OFTRL’s regret} \label{app:OptRL}

For completeness, we replicate the proof in \cite{ALLW18}
about the regret bound of \OFTRL in this section.

\textbf{Theorem 16} of  [\cite{ALLW18}]
\textit{
Let $\theta_{t}$ be the loss vector in round $t$. 
Let the update of \OFTRL be
$\textstyle x_t  = \arg \min_{x \in \K} \langle x , L_{t-1} + m_t \rangle + \frac{1}{\eta} R(x)$,
where $m_{t}$ is the learner’s guess of the loss vector in round $t$ and $R(x)$ is a $\beta$-strongly 
convex function.
Denote the update of standard \textsc{FTRL} as
$\textstyle z_t  = \arg \min_{x \in \K} \langle x , L_{t-1} \rangle + \frac{1}{\eta} R(x)$.
Also, $\textstyle z_{1} = \arg \min_{x \in \K}  R(x)  $.
Then,
\OFTRL (\ref{optFTRL}) has regret
\begin{equation} \label{opt-ftrl}
\textstyle \text{Regret}:= \sum_{t=1}^{T} \langle x_t - x^{*}, \theta_{t} \rangle 
\leq \frac{ R(x^*) - R(z_1) - D_T }{\eta} +
\sum_{t=1}^{T} \frac{\eta}{\beta} \| \theta_t - m_{t} \|^2_* ,
\end{equation}
where $D_{T} = \sum_{{t=1}}^{T} \frac{\beta}{2} \| x_t - z_{t} \|^2 
+ \frac{\beta}{2} \| x_t - z_{t+1} \|^2$,
$z_{t} = \argmin_{x \in \K} \langle x, L_{{t-1}} \rangle + \frac{1}{\eta} R(x)$,
and $x_{t} = \arg \min_{x \in \K} \langle x , L_{t-1} + m_t \rangle + \frac{1}{\eta} R(x)$.
}
\begin{proof}

Define $z_{t}= \argmin_{x \in \K} \langle x, L_{{t-1}} \rangle + \frac{1}{\eta} R(x)$
as the update of the standard \textsc{Follow-the-Regularized-Leader}.
We can re-write the regret as
\begin{equation}
\begin{aligned}
\textstyle \text{Regret}:= \sum_{t=1}^{T} \langle x_t - x^{*}, \theta_{t} \rangle
= \sum_{t=1}^{T} \langle x_t - z_{t+1} , \theta_t - m_t \rangle + 
 \sum_{t=1}^{T}  \langle x_t - z_{t+1} , m_t \rangle + \langle z_{t+1} - x^*, \theta_t \rangle 
\end{aligned}
\end{equation}
Let us analyze the first sum
\begin{equation}
\textstyle \sum_{t=1}^{T} \langle x_t - z_{t+1} , \theta_t - m_t \rangle.
\end{equation}
Now using Lemma 17 of \cite{ALLW18} (which is also stated below)
with $x_{1}=x_{t}$, $u_{1}=\sum_{{s=1}}^{{t-1}} \theta_{s} + m_{t}$ and $x_{2}=z_{t+1}$, 
$u_{2}=\sum_{{s=1}}^{{t}} \theta_{s} $ in the lemma,
we have 
\begin{equation} \label{qq1}
\textstyle \sum_{t=1}^{T} \langle x_t - z_{t+1} , \theta_t - m_t \rangle
\leq \sum_{t=1}^{T} \|  x_t - z_{t+1} \| \| \theta_t - m_t\|_* 
\leq \sum_{t=1}^{T} \frac{\eta}{\beta} \| \theta_t - m_t \|^2_*.
\end{equation}
For the other sum,
\begin{equation}
\textstyle \sum_{t=1}^{T}  \langle x_t - z_{t+1} , m_t \rangle + \langle z_{t+1} - x^*, \theta_t \rangle,
\end{equation}
we are going to show that, for any $T\geq0$, it is upper-bounded by
$\frac{ R(x^*) - R(z_1) - D_T }{\eta}$,
which holds for any $x^* \in \K$, where $D_{T} = \sum_{{t=1}}^{T} \frac{\beta}{2} \| x_t - z_{t} \|^2 
+ \frac{\beta}{2} \| x_t - z_{t+1} \|^2$.
For the base case $T=0$, we see that
\begin{equation}
\textstyle \sum_{t=1}^{0}  \langle x_t - z_{t+1} , m_t \rangle + \langle z_{t+1} - x^*, \theta_t \rangle 
= 0 
\leq \frac{  R(x^*) - R(z_1) - 0 }{\eta},
\end{equation}
as $z_1 = \arg \min_{x \in \K} R(x)$.

Using induction,
assume that it also holds for $T-1$ for a $T \geq 1$.
Then, we have
\begin{equation}
\begin{aligned}
& \textstyle \sum_{t=1}^{T}  \langle x_t - z_{t+1} , m_t \rangle + \langle z_{t+1} , \theta_t \rangle
\\& \textstyle \overset{(a)}{\leq} \langle x_T - z_{T+1}, m_T \rangle + \langle z_{T+1}, \theta_T \rangle
+ \frac{ R(z_T) - R(z_1) - D_{T-1} }{\eta} + \langle z_{T}, L_{T-1} \rangle
\\& \textstyle \overset{(b)}{\leq}
\langle x_T - z_{T+1}, m_T \rangle + \langle z_{T+1}, \theta_T \rangle
+ \frac{ R(x_T) - R(z_1) - D_{T-1} - \frac{\beta}{2} \| x_T - z_{T} \|^2 }{\eta} + \langle x_T, L_{T-1} \rangle
\\& \textstyle = \langle z_{T+1}, \theta_T - m_T \rangle 
+ \frac{ R(x_T) - R(z_1) -D_{T-1} - \frac{\beta}{2} \| x_T - z_{T} \|^2 }{\eta} + \langle x_T, L_{T-1} + m_T \rangle
\\& \textstyle \overset{(c)}{\leq}  \langle z_{T+1}, \theta_T - m_T \rangle 
+ \frac{ R(z_{T+1}) - R(z_1) - D_{T-1} - \frac{\beta}{2} \| x_T - z_{T} \|^2 
- \frac{\beta}{2} \| x_T - z_{T+1} \|^2
}{\eta}
\\& \textstyle + \langle z_{T+1}, L_{T-1} + m_T \rangle
\\& \textstyle = \langle z_{T+1}, L_T \rangle 
+ \frac{ R(z_{T+1}) - R(z_1) - D_{T} }{\eta}
\\ & \textstyle \overset{(d)}{\leq} 
\langle x^*, L_T \rangle 
+ \frac{ R(x^*) - R(z_1) - D_{T} }{\eta},
\end{aligned}
\end{equation}
where (a) is by induction such that the inequality holds at $T-1$ for any $x^{*} \in \K$ including $x^{*} = z_{T}$,
 (b) and (c) are by strong convexity so that
\begin{equation}
\textstyle \langle z_{T}, L_{T-1} \rangle + \frac{R(z_T)}{\eta}
\leq \langle x_{T}, L_{T-1} \rangle + \frac{R(x_T)}{\eta} - \frac{\beta}{2 \eta} \| x_T - z_T  \|^2,
\end{equation}
and
\begin{equation}
\textstyle \langle x_{T}, L_{T-1} + m_T \rangle + \frac{R(x_T)}{\eta}
\leq \langle z_{T+1}, L_{T-1} + m_T \rangle + \frac{R(z_{T+1})}{\eta} - \frac{\beta}{2 \eta} \| x_T - z_{T+1}  \|^2,
\end{equation}
and (d) is because $z_{T+1}$ is the optimal point of 
$\argmin_x \langle x , L_T \rangle + \frac{R(x)}{\eta}$. We’ve completed the induction.

\end{proof}

\textbf{Lemma 17 of} [\cite{ALLW18}] 
\textit{
Denote $x_1 = \argmin_{x} \langle x, u_1 \rangle + \frac{1}{\eta} R(x)$
and $x_{2} = \argmin_{x} \langle x, u_2 \rangle + \frac{1}{\eta} R(x)$
for a $\beta$-strongly convex function $R(\cdot)$ with respect to a norm $\| \cdot \|$.
We have $\| x_{1} - x_{2} \| \leq \frac{\eta}{\beta} \| u_1 - u_2\|_{*}$.} 

\section{Proof of Theorem~\ref{thm:acc_linear}} \label{app:acc_linear}
\textbf{Theorem~\ref{thm:acc_linear}}
\textit{
For the game 
$g(x,y):= \langle x, y \rangle - \tilde{f}^*(y) + \frac{\mu\|x\|^{2}_{2} }{2}$, if the y-player plays \textsc{OptimisticFTL}
and the x-player plays \textsc{BeTheRegularizedLeader}:
$x_t \leftarrow \arg\min_{{x \in \XX}} \sum_{s=0}^t \alpha_{s} \ell_{s}(x)$,
where $\alpha_0 \ell_{0}(x):= \alpha_{0} \frac{\mu  \| x\|^2_2}{2}$, then 
the weighted average $(\xav_T, \yav_T)$ would be $O(\exp(-\frac{T}{\sqrt{\kappa}}))$-approximate equilibrium
of the game,
where the weights $\frac{\alpha_t}{\tilde{A_t} } = \frac{1}{\sqrt{6 \kappa}}$.
This implies that $f(\xav_T) - \min_{x \in \XX} f(x) = O(\exp(-\frac{T}{\sqrt{\kappa}})).$
}

\begin{proof}

From Lemma~\ref{lem:yregretbound},
we know that the y-player’s regret by \textsc{OptimisticFTL} is 
    \begin{eqnarray*}
         \textstyle \sum_{t=1}^{T} \alpha_t \ell_t(\yof_t) - \alpha_t \ell_t(y^*) 
        & \leq &  \textstyle \sum_{t=1}^T \delta_t(\yof_t) - \delta_t(\yftl_{t+1}) \\
        & = & \textstyle  \sum_{t=1}^T \alpha_t \langle x_{t-1} - x_{t}, \yof_t - \yftl_{t+1} \rangle \\
        \text{(Eqns. \ref{eq:haty}, \ref{eq:tildey})} \quad \quad
        & = & \textstyle  \sum_{t=1}^T \alpha_t \langle x_{t-1} - x_{t}, \nabla \tilde{f}(\xof_t) - \nabla \tilde{f}(\xav_t) \rangle \\
        \text{(H\"older's Ineq.)} \quad \quad
        & \leq & \textstyle  \sum_{t=1}^T \alpha_t \| x_{t-1} - x_{t}\| \| \nabla \tilde{f}(\xof_t) - \nabla \tilde{f}(\xav_t) \| \\
        & = & \textstyle  \sum_{t=1}^T \alpha_t \| x_{t-1} - x_{t}\| \| \nabla f(\xof_t) -\mu \xof_t - \nabla \tilde{f}(\xav_t) + \mu \xav_t \| \\
                        \text{(triangle inequality)} \quad \quad
        & \leq & \textstyle  \sum_{t=1}^T \alpha_t \| x_{t-1} - x_{t}\| ( \| \nabla f(\xof_t) - \nabla \tilde{f}(\xav_t) \| +  \mu \| \xav_t - \xof_t \| ) \\
                \text{($L$-smoothness and $L\geq \mu$)} \quad \quad
        & \leq & \textstyle  2 L \sum_{t=1}^T \alpha_t \| x_{t-1} - x_{t}\| \|\xof_t - \xav_t \| \\
        \text{(Eqn. \ref{eq:xdiffs})} \quad \quad
        & = & \textstyle 2 L \sum_{t=1}^T \frac{\alpha_t^2}{A_t} \| x_{t-1} - x_{t}\| \|x_{t-1} - x_{t} \|
    \end{eqnarray*}

Therefore,
\begin{equation} \label{reg_y_accH}
\textstyle
\regret{y} \leq 2 L \sum_{t=1}^T \frac{\alpha_t^2}{A_t} \|x_{t-1} - x_t \|^2.
\end{equation}

For the x-player, its loss function in round $t$ is $\alpha_{t} \ell_{t}(x):= \alpha_t ( \mu \phi(x) + \langle x, y_t \rangle) $, where $\phi(x):= \frac{1}{2}\|x\|^{2}_{2}$.
Assume the x-player plays \textsc{BeTheRegularizedLeader},
\begin{equation}
x_t \leftarrow \arg\min_{{x \in \XX}} \sum_{s=0}^t \alpha_{s} \ell_{s}(x),
\end{equation}
where $\alpha_0 \ell_{0}(x):= \alpha_{0} \mu \phi(x)$.
Denote 
\begin{equation}
\tilde{A}_{t}:= \sum_{s=0}^{t} \alpha_{s}.
\end{equation}
Notice that this is different from $A_{t}:= \sum_{s=1}^{t} \alpha_{s}$. 
Then, its regret is (proof is on the next page)
\begin{equation} \label{reg_x_acc0}
\begin{aligned}
\textstyle
\regret{x} := \sum_{t=1}^T \alpha_t \ell_t(x_t) - \alpha_t \ell_t( x^*)
\leq \alpha_0 \mu L_0 \| x^* - x_0 \| - \sum_{t=1}^T \frac{\mu \tilde{A}_{t-1}}{2} \|x_{t-1} - x_t \|^2,
\end{aligned}
\end{equation} 
where $L_{0}$ is the Lipchitz constant of the 1-strongly convex function $\phi(x)$
and $x_{0} = \arg\min_{x} \phi(x)$.

Summing (\ref{reg_y_accH}) and (\ref{reg_x_acc0}), we have
\begin{equation}
\begin{aligned}
& \regret{y} + \regret{x} \leq \alpha_0 \mu L_0  \|x^* - x_0\| + \sum_{t=1}^T ( \frac{2 L \alpha_t^2}{A_t} - \frac{\mu \tilde{A}_{t-1}}{2} ) \|x_{t-1} - x_t \|^2.
\end{aligned}
\end{equation}
We want to let the distance terms cancel out.
\begin{equation}
\frac{2 L \alpha_t^2}{ \tilde{A_t} - a_0} - \frac{\mu \tilde{A}_{t-1}}{2}  \leq 0,
\end{equation} 
which is equivalent to 
\begin{equation}
\begin{aligned}
   & 4 L \alpha_t^2 \leq \mu \tilde{A_t} \tilde{A}_{t-1} - \mu \alpha_0 \tilde{A}_{t-1}.
\\ & 4 L \frac{\alpha_t^2}{\tilde{A_t}^2 } \leq \mu \frac{\tilde{A}_{t-1} }{\tilde{A_t} } 
-  \mu \alpha_0 \frac{\tilde{A}_{t-1}}{\tilde{A_t} } \frac{1}{\tilde{A_t} }
\\ & 4 L \frac{\alpha_t^2}{\tilde{A_t}^2 } \leq \mu ( 1 - \frac{\alpha_0}{\tilde{A_t} } ) 
(1 -\frac{\alpha_t}{\tilde{A_t} })
\end{aligned}
\end{equation}
Let us denote the constant $\theta := \frac{\alpha_t}{\tilde{A_t} } > 0$.
\begin{equation}
\theta^2 + \frac{\mu}{4L} ( 1 - \frac{\alpha_0}{\tilde{A_t} } ) \theta
- \frac{\mu}{4L}( 1 - \frac{\alpha_0}{\tilde{A_t} }  ) \leq 0.
\end{equation}
Notice that $0 < \frac{\alpha_0}{\tilde{A_t} } \leq  1$.  It suffices to show that
\begin{equation}
\theta^2 + \frac{\mu}{4L} ( 1 - \frac{\alpha_0}{\tilde{A_t} } ) \theta
- \frac{\mu}{4L} \leq 0.
\end{equation}
Yet, we would expect that $\frac{\alpha_0}{\tilde{A_t} }$ is a decreasing function of $t$, so
it suffices to show that 
\begin{equation}
\theta^2 + \frac{\mu}{4L} ( 1 - \frac{\alpha_0}{\tilde{A_1} } ) \theta
- \frac{\mu}{4L} \leq 0,
\end{equation}
which is equivalent to
\begin{equation}
\begin{aligned}
& \theta^2 + \frac{\mu}{4L}  \frac{\alpha_1}{ \tilde{A_1}  } \theta - \frac{\mu}{4L} \leq 0
\\ & \theta^2 ( 1 + \frac{\mu}{4L}) - \frac{\mu}{4L} \leq 0.
\end{aligned}
\end{equation}
It turns out that $\theta= \sqrt{ \frac{\mu}{6L}} = \frac{1}{\sqrt{6 \kappa}}$ satisfies the above inequality, combining the fact that 
$\frac{\mu}{L} \leq 1$.

Therefore, the optimization error $\epsilon$ after $T$ iterations is
\begin{equation}
\begin{aligned}
& \epsilon \leq
\frac{ \regret{y} + \regret{x} }{A_T} 
\leq \frac{1}{A_1} \frac{A_1}{A_2} \cdots \frac{A_{T-1}}{A_T}
( \alpha_0 \mu L_0 \| x^* - x_0 \|  ) 
\\ & =  \frac{1}{A_1} (1 - \frac{\alpha_2}{A_2} ) \cdots (1 - \frac{\alpha_T}{A_T} )
( \alpha_0 \mu L_0 \| x^* - x_0 \|  ) 
\\ & \leq \frac{1}{A_1} (1 - \frac{\alpha_2}{\tilde{A}_2} ) \cdots (1 - \frac{\alpha_T}{\tilde{A}_T} )
( \alpha_0 \mu L_0 \| x^* - x_0 \|  ) 
\\ & \leq (1- \frac{1}{\sqrt{6 \kappa}} )^{T-1} \frac{\alpha_0 \mu L_0}{A_1} \| x^* - x_0 \|.
\end{aligned}
\end{equation}
which is $O( (1 - \frac{1}{\sqrt{6\kappa}})^{T}) = O( \exp( - \frac{1}{\sqrt{6\kappa}} T))$. 

\end{proof}

\begin{proof} (of (\ref{reg_x_acc0}))
First, we are going to use induction to show that
\begin{equation} \label{key:stc0}
\sum_{t=0}^{\tau} \alpha_t \ell_t(x_t) - \alpha_t \ell_t( x^*) \leq D_{\tau},
\end{equation}
for any $x^{*} \in \XX$,
where $D_{\tau}:= - \sum_{t=1}^{\tau} \frac{\mu \tilde{A}_{t-1}}{2} \|x_{t-1} - x_t \|^2. $

For the base case $t=0$, we have
\begin{equation}
\alpha_0 \mu \phi(x_0) - \alpha_0 \mu \phi(x^*)  \leq 0 = D_0,
\end{equation}
where $x_{0}$ is defined as $x_{0} = \arg\min_{{x \in \XX}} \alpha_0 \mu \phi(x)$.

Now suppose it holds at $t= \tau-1$.
\begin{equation}
\begin{aligned}
\sum_{t=0}^{\tau} \alpha_t \ell_t(x_t) 
& \overset{(a)}{\leq} D_{\tau-1} + \alpha_{\tau} \ell_{\tau} (x_{\tau})
+ \sum_{t=0}^{\tau-1} \alpha_t \ell_t(x_{\tau-1})
\\ &
\overset{(b)}{\leq} D_{\tau-1} + \alpha_{\tau} \ell_{\tau} (x_{\tau})
+ \sum_{t=0}^{\tau-1} \alpha_t \ell_t(x_{\tau}) - \frac{ \tilde{A}_{\tau-1} \mu }{2} \| x_{\tau-1} - x_{\tau}  \|^2
\\ & = D_{\tau-1} + \sum_{t=0}^{\tau} \alpha_t \ell_t(x_{\tau}) - \frac{ \tilde{A}_{\tau-1} \mu }{2} \| x_{\tau-1} - x_{\tau}  \|^2
\\ & = D_{\tau} +  \sum_{t=0}^{\tau} \alpha_t \ell_t(x_{\tau}) 
\\ & \leq D_{\tau} +  \sum_{t=0}^{\tau} \alpha_t \ell_t(x^*), 
\end{aligned}
\end{equation}
for any $x^{*} \in \XX$,
where $(a)$ we use the induction and we let the point $x^{*}= x_{{\tau-1}}$
and $(b)$ is by the strongly convexity and that 
$x_{\tau-1} = \arg\min_{x}  \sum_{t=0}^{\tau-1} \alpha_t \ell_t(x)$
so that $\sum_{t=0}^{\tau-1} \alpha_t \ell_t(x_{\tau-1}) 
\leq  \sum_{t=0}^{\tau-1} \alpha_t \ell_t(x_{\tau}) - \frac{ \tilde{A}_{\tau-1} \mu}{2} \| x_{\tau-1} - x_{\tau} \|^{2}$ as $\sum_{t=0}^{\tau-1} \alpha_t \ell_t(x)$ is at least $\frac{ \tilde{A}_{\tau-1} \mu}{2}$-strongly convex. We have completed the proof of (\ref{key:stc0}).
By (\ref{key:stc0}), we have
\begin{equation} \label{reg_x_acc}
\begin{aligned} 
& \textstyle
\regret{x} := \sum_{t=1}^T \alpha_t \ell_t(x_t) - \alpha_t \ell_t( x^*)
\leq \alpha_0 \mu \phi(x^*) - \alpha_0 \mu \phi(x_0) - \sum_{t=1}^T \frac{\mu \tilde{A}_{t-1}}{2} \|x_{t-1} - x_t \|^2.
\\ & \leq \alpha_0 \mu L_0 \| x_0 - x^*\| - \sum_{t=1}^T \frac{\mu \tilde{A}_{t-1}}{2} \|x_{t-1} - x_t \|^2,
\end{aligned}
\end{equation} 
where we assume that $\phi(\cdot)$ is $L_{0}$-Lipchitz.

\end{proof}

\section{Analysis of Accelerated Proximal Method} \label{app:accProx}

First, we need a stronger result.  

\textbf{Lemma} [Property 1 in \cite{T08}] 
\textit{
For any proper lower semi-continuous convex function $\theta(x)$,
let $x^+ =  \argmin_{x \in \K} \theta(x) + V_{c}(x)$. 
Then, it satisfies that for any $x^* \in \K$,
\begin{equation} 
\textstyle \theta(x^+) - \theta(x^*) \leq V_{c}(x^*) - V_{x^+}(x^*) - V_{c}(x^+).
\end{equation}
}

\begin{proof}
The statement and its proof has also appeared in
\cite{CT93} and \cite{LLM11}.
For completeness, we replicate the proof here.
Recall that the Bregman divergence with respect to the distance generating function $\phi(\cdot)$ at 
a point $c$ is: $V_{c}(x):= \phi(x) - \langle \nabla \phi(c), x - c  \rangle  - \phi(c).$ 

Denote $F(x):= \theta(x) + V_{c}(x)$.
Since $x^+$ is the optimal point of 
$\argmin_{x \in K} F(x)\textstyle,$
by optimality,
\begin{equation} \label{m1}
\textstyle \langle x^* - x^+ ,  \nabla F(x^+) \rangle =  
\langle x^* - x^+ , \partial \theta(x^+) + \nabla \phi (x^+) - \nabla \phi(c) \rangle  \geq 0,
\end{equation}
for any $x^{*} \in K$.

Now using the definition of subgradient, we also have
\begin{equation} \label{m2}
\textstyle \theta(x^*) \geq \theta(x^+) + \langle \partial \theta(x^+) , x^* - x^+ \rangle.
\end{equation}
By combining (\ref{m1}) and (\ref{m2}), we have
\begin{equation} 
\begin{aligned}
 \textstyle   \theta(x^*) & \textstyle \geq \theta(x^+) + \langle \partial \theta(x^+) , x^* - x^+ \rangle.
\\ &  \textstyle \geq \theta(x^+) + \langle x^* - x^+ , \nabla \phi (c) - \nabla \phi(x^+) \rangle.
\\  & \textstyle = \theta(x^+) - \{ \phi(x^*) - \langle \nabla \phi(c), x^* - c  \rangle  - \phi(c) \}
+ \{ \phi(x^*) - \langle \nabla \phi(x^+), x^* - x^+  \rangle  - \phi(x^+) \}
\\  & \textstyle + \{ \phi(x^+) - \langle \nabla \phi(c), x^+ - c  \rangle  - \phi(c) \}
\\  & \textstyle  = \theta(x^+) - V_c(x^*) + V_{x^+}(x^*) +  V_c(x^+)
\end{aligned}
\end{equation}
\end{proof}

Recall \MD’s update
$x_{t} = \argmin_{x}   \gamma_t ( \alpha_t h_t(x) ) + V_{x_{t-1}}(x)$,
where $h_t(x) = \langle x,  y_t \rangle + \psi(x)$.
Using the lemma with $\theta(x) = \gamma_t ( \alpha_t h_t(x) )$, $x^+= x_t$ and $c=x_{t-1}$  we have that
\begin{equation} \label{ttb1}
\textstyle \gamma_t ( \alpha_t h_t(x_t) ) - \gamma_t ( \alpha_t h_t(x^*) ) =  \theta(x_t) - \theta(x^*) \leq V_{x_{t-1}}(x^*) - V_{x_t}(x^*) - V_{x_{t-1}}(x_t).
\end{equation}
Therefore, we have that
\begin{equation}
\begin{aligned}
& \textstyle  \regret{x} := \sum_{t=1}^T  \alpha_t  h_t(x_t) - \min_{x \in \XX} \sum_{t=1}^T  \alpha_t  h_t(x)
\\& \textstyle  \overset{(\ref{ttb1})}{\leq}  \sum_{t=1}^{T} \frac{1}{\gamma_t} \big( V_{x_{t-1}}(x^*) - V_{x_t}(x^*) - V_{x_{t-1}}(x_t) \big) 
\\ & \textstyle
=  \frac{1}{\gamma_1} V_{x_{0}}(x^*) - \frac{1}{\gamma_T} v_{x_T}(x^*)  + \sum_{t=1}^{T-1} ( \frac{1}{\gamma_{t+1}} - \frac{1}{\gamma_t} ) V_{x_t}(x^*) - \frac{1}{\gamma_t} V_{x_{t-1}}(x_t) 
\\ & \textstyle
\overset{(a)}{ \leq} \frac{1}{\gamma_1} D + \sum_{t=1}^{T-1} ( \frac{1}{\gamma_{t+1}} - \frac{1}{\gamma_t} ) D - \frac{1}{\gamma_t} V_{x_{t-1}}(x_t) 
= \frac{D}{\gamma_{T}} -  \sum_{t=1}^{T} \frac{1}{\gamma_t} V_{x_{t-1}}(x_t) 
\\ & \textstyle \overset{(b)}{ \leq} \frac{D}{\gamma_{T}} - \sum_{t=1}^{T} \frac{1}{2 \gamma_t} \| x_{t-1} - x_t \|^2,
\end{aligned}
\end{equation}
where $(a)$ holds since the sequence $\{\gamma_t\}$ is non-increasing and $D$ upper bounds the divergence terms, and $(b)$ follows from the strong convexity of $\phi$, which grants $V_{x_{t-1}}(x_t) \geq \frac 1 2 \| x_t - x_{t-1}\|^2$. Now we see that following the same lines as the proof in Section~\ref{analysis:meta}.
We get that $\bar{x}_{T}$ is an $O(\frac{1}{T^2})$ approximate optimal solution.

\section{Accelerated \FW} \label{app:accFW}

\begin{algorithm}[h] 
  \caption{A new FW algorithm [\cite{ALLW18}]}
  \label{alg:FW}
  \begin{algorithmic}[1]
    \STATE In the weighted loss setting of Algorithm~\ref{alg:game}:
    \FOR{$t= 1, 2, \dots, T$}
    \STATE \quad $y$-player uses \textsc{OptimisitcFTL} as $\alg^x$: $y_t = \nabla f( \xof_{t})$.
    \STATE \quad $x$-player uses \BTRL  with $R(X):= \frac{1}{2} \g(x)^{2}$ as $\alg^x$:
    \STATE \qquad Set $(\hat{x}_t, \rho_t) = \underset{x \in \K, \rho \in[0,1] }{\argmin} \sum_{s=1}^t \rho \langle  x, \alpha_s y_s \rangle  +  \frac{1}{\eta} \rho^2$ and play $x_t = \rho_t \hat{x}_t$.   
    \ENDFOR
  \end{algorithmic}
\end{algorithm}

\cite{ALLW18} proposed a \FW like algorithm that not only requires a linear oracle but also enjoys $O(1/T^2)$ rate on all the known examples 
of strongly convex constraint sets that contain the origin,
like $l_p$ ball and Schatten $p$ ball with $p \in (1,2]$.
Their analysis requires the assumption that the underlying function is also strongly-convex to 
get the fast rate.
To describe their algorithm, denote $\K$ be any closed convex set that contains the origin.
Define ``gauge function'' of $\K$ as $\g(x) := \inf \{ c \geq 0: \frac{x}{c} \in \K \}$.
Notice that, for a closed convex $\K$ that contains the origin, 
$\K = \{ x \in \reals^d: \g(x)\leq 1 \}$. Furthermore, the boundary points on $\K$ satisfy $\g(x)=1$.

\cite{ALLW18} showed that the squared of a gauge function is strongly convex on the underlying $\K$
for all the known examples of strongly convex sets that contain the origin.
Algorithm~\ref{alg:FW} is the algorithm.
Clearly, Algorithm~\ref{alg:FW} is an instance of the meta-algorithm.
We want to emphasize again that our analysis does not need the function $f(\cdot)$ to be strongly convex 
to show $O(1/T^2)$ rate. We’ve improved their analysis.

\clearpage

\section{Proof of Theorem~\ref{thm:convergence}} \label{app:thm:convergence}

For completeness, we replicate the proof by \cite{ALLW18} here.

\textbf{Theorem~\ref{thm:convergence}} 
\textit{
  Assume a $T$-length sequence $\balpha$ are given. Suppose in Algorithm~\ref{alg:game} the online learning algorithms $\alg^x$ and $\alg^y$ have the $\balpha$-weighted average regret $\avgregret{x}$ and $\avgregret{y}$ respectively. Then the output  $(\bar{x}_{T},\bar{y}_{T})$ is an $\epsilon$-equilibrium for $g(\cdot, \cdot)$, with
$    \epsilon = \avgregret{x} + \avgregret{y}.$
} 

\begin{proof}
Suppose that the loss function of the $x$-player in round $t$ is $\alpha_t h_t(\cdot) : \XX \to \reals$, where $h_t(\cdot) := g(\cdot, y_t)$. The $y$-player, on the other hand, observes her own sequence of loss functions $\alpha_t  \ell_t(\cdot) : \YY \to \reals$, where $\ell_t(\cdot) := -  g(x_t, \cdot)$.

\begin{eqnarray}
\frac{1}{\sum_{s=1}^T \alpha_s}   \sum_{t=1}^T \alpha_t g(x_t, y_t) 
  & = & \frac{1}{\sum_{s=1}^T \alpha_s} \sum_{t=1}^T  - \alpha_t  \ell_t(y_t)  \notag \\
  \text{} \; & = & 
    - \frac{1}{\sum_{s=1}^T \alpha_s} \inf_{y \in \YY} \left\{ \sum_{t=1}^T  \alpha_t  \ell_t(y) \right\} - \frac{ \regret{y} }{  \sum_{s=1}^T \alpha_s } \notag \\
  \; & = &
    \sup_{y \in \YY} \left\{ \frac{1}{\sum_{s=1}^T \alpha_s} \sum_{t=1}^T  \alpha_t g( x_t , y ) \right\} - \avgregret{y}  \notag \\
  \text{(Jensen)} \;  & \geq & 
    \sup_{y \in \YY} g\left({\textstyle \frac{1}{\sum_{s=1}^T \alpha_s} \sum_{t=1}^T  \alpha_t x_t }, y \right)  - \avgregret{y} 
     \label{eq:ylowbound}  \\
  \text{} \;  & = & 
    \sup_{y \in \YY} g\left({\textstyle \xav_T }, y \right)  - \avgregret{y} 
     \label{eq:ylowbound}  \\
  & \geq & \inf_{x \in \XX} \sup_{y \in \YY} g\left( x , y \right) - \avgregret{y}  \notag
\end{eqnarray}

Let us now apply the same argument on the right hand side, where we use the $x$-player's regret guarantee.
\begin{eqnarray}
\frac{1}{\sum_{s=1}^T \alpha_s}  \sum_{t=1}^T  \alpha_t g(x_t, y_t) & = & \frac{1}{\sum_{s=1}^T \alpha_s} \sum_{t=1}^T  \alpha_t h_t(x_t) \notag \\
  & = & \left\{ 
    \sum_{t=1}^T \frac{1}{\sum_{s=1}^T \alpha_s} \alpha_t h_t(x) \right \} + \frac{ \regret{x} }{  \sum_{s=1}^T \alpha_s }  \notag \\
  & = &  \left\{  \sum_{t=1}^T \frac{1}{\sum_{s=1}^T \alpha_s} \alpha_t g(x^*, y_t) \right \} + \avgregret{x} \notag \\
  & \leq &   
    g\left(x^*,{ \textstyle \sum_{t=1}^T \frac{1}{\sum_{s=1}^T \alpha_s} \alpha_t y_t}\right) + \avgregret{x} 
    \label{eq:xupbound} \\
  & = & 
    g\left(x^*,{ \textstyle \yav_T}\right) + \avgregret{x} 
    \label{eq:xupbound} \\
  & \leq & \sup_{y \in \YY}  g(x^*,y) + \avgregret{x}  \notag
\end{eqnarray}
Note that $\sup_{y \in \YY}  g(x^*,y) = f(x^*)$ be the definition of the game $g(\cdot,\cdot)$ and by Fenchel conjugacy, hence we can conclude that $\sup_{y \in \YY} g(x^*,y) = \inf_{x \in \XX} \sup_{y \in \YY} g(x,y) = V^* = \sup_{y \in \YY} \inf_{x \in \XX}  g(x,y)$.
Combining (\ref{eq:ylowbound}) and (\ref{eq:xupbound}), we see that:
\begin{align*}
    \sup_{y \in \YY} g\left({\textstyle \xav_T}, y \right)  - \avgregret{y}  \le \inf_{x \in \XX}  
g\left(x,{ \textstyle \yav_T }\right) + \avgregret{x} 
\end{align*}
which implies that $(\bar{x}_{T}, \bar{y}_{T})$ is an $\epsilon =\avgregret{x} + \avgregret{y} $ equilibrium.
\end{proof}

\end{document}